%% file: main.tex
\documentclass{article}

\usepackage[preprint]{neurips_2026}
\makeatletter
\renewcommand{\@notice}{\enlargethispage{2\baselineskip}}
\makeatother

\usepackage[utf8]{inputenc} % allow utf-8 input
\usepackage[T1]{fontenc}    % use 8-bit T1 fonts
\usepackage{hyperref}       % hyperlinks
\usepackage{url}            % simple URL typesetting
\usepackage{booktabs}       % professional-quality tables
\usepackage{amsfonts}       % blackboard math symbols
\usepackage{nicefrac}       % compact symbols for 1/2, etc.
\usepackage{microtype}      % microtypography
\usepackage{xcolor}         % colors

\usepackage{subfiles} 
\usepackage{multirow}
\usepackage{tabularx}
\usepackage{graphicx}
\usepackage{caption}
\usepackage{enumitem}
\usepackage{subfig} % 提供table/figure 并排排列功能
\usepackage{wrapfig} % 图片浮动
\usepackage{makecell} % 引入makecell宏包, 表格内换行
\usepackage{amsmath}
\usepackage{amsthm}
\newtheorem{proposition}{Proposition}
\newtheorem{corollary}{Corollary}
\usepackage{minitoc}
\usepackage{titletoc}

\usepackage{algorithm}
\usepackage{algpseudocode}
\usepackage{booktabs}
\usepackage{xcolor}
\usepackage[table]{xcolor}

\title{Tikhonov-Stabilized B\'ezier Representation Forecasting for Training-free Diffusion Acceleration}

\author{%
  \mdseries Lei Zhu\textsuperscript{1,}\thanks{Equal contribution.}, Mujie Lin\textsuperscript{1,*},
  Ruochong Zheng\textsuperscript{1}, Guangyi Wang\textsuperscript{2} \\
  \mdseries Hao Li\textsuperscript{1}, Peng Jin\textsuperscript{1},
  Chang Liu\textsuperscript{3}, Jie Chen\textsuperscript{1,}\thanks{Corresponding author.} \\[0.5em]
  \textsuperscript{1}School of Electronic and Computer Engineering, Peking University \\
  \textsuperscript{2}School of Informatics, Xiamen University \\
  \textsuperscript{3}Department of Automation and BNRist, Tsinghua University
}

\begin{document}
\maketitle

\begin{abstract}
Diffusion models, particularly Diffusion Transformers, achieve strong image and video generation quality but remain expensive at inference time due to repeated denoiser evaluations. Feature caching offers a deployment-friendly acceleration strategy by reusing or predicting intermediate representations without retraining the generator or modifying the sampler. However, existing cache-then-forecast methods often rely on Taylor-style finite-difference extrapolation, which can become unstable over longer cache intervals, and typically predict local module outputs whose errors may accumulate through subsequent denoiser blocks. We propose \textbf{BeziCast}, a training-free diffusion acceleration framework that forecasts output-proximal denoising representations using low-order B\'ezier trajectories. BeziCast estimates B\'ezier control points via Tikhonov-stabilized fitting, providing a smooth, capacity-controlled temporal parameterization that avoids explicit high-order derivative estimation and decouples trajectory capacity from cache-query density. To moderate aggressive extrapolation, we further derive equivalent forecasting weights and introduce a convex-hull-inspired guardrail that detects high-risk predictions and softly pulls them toward a conservative simplex-projected estimate. 
Extensive evaluations across advanced image and video diffusion models demonstrate the effectiveness of BeziCast. In particular, BeziCast accelerates FLUX.1 by up to \(4.79\times\) and HunyuanVideo by up to \(4.11\times\), while preserving substantially better generation quality than competing acceleration baselines.
\end{abstract}

\subfile{Contents/1intro}
\subfile{Contents/2related}
\subfile{Contents/3method}
\subfile{Contents/4exp}
% \subfile{Contents/5discussion}

\newpage
\bibliographystyle{plain}
\bibliography{reference}

%%%%%%%%%%%%%%%%%%%%%%%%%%%%%%%%%%%%%%%%%%%%%%%%%%%%%%%%%%%%

\subfile{Contents/appendix}
% Technical appendices with additional results, figures, graphs, and proofs may be submitted with the paper submission before the full submission deadline (see above). You can upload a ZIP file for videos or code, but do not upload a separate PDF file for the appendix. There is no page limit for the technical appendices. 

% Note: Think of the appendix as ``optional reading'' for reviewers. The paper must be able to stand alone without the appendix; for example, adding critical experiments that support the main claims to an appendix is inappropriate. 

%%%%%%%%%%%%%%%%%%%%%%%%%%%%%%%%%%%%%%%%%%%%%%%%%%%%%%%%%%%%

% \newpage
% \input{checklist.tex}

\end{document}

%% file: Contents/1intro.tex
\section{Introduction}
\vspace{-0.25cm}

Diffusion models\cite{ho2020denoising}, especially large generative models based on Diffusion Transformer\cite{peebles2023scalable,esser2024scaling,zhu2026omnihuman,kong2024hunyuanvideo}, have become a standard approach for high-fidelity image and video generation. Their sampling process, however, still requires repeated denoiser evaluations across many timesteps, and this inference cost remains a major obstacle to large-scale deployment.

A central question in diffusion acceleration is how to approximate the original multi-step sampling process under a limited NFE budget\cite{lu2022dpm,zhao2023unipc}. Existing approaches can be broadly grouped into three directions. Solver-based methods improve the sampler through higher-order updates\cite{lu2022dpm,lu2025dpm,zhao2023unipc}, timestep schedules\cite{karras2022elucidating}, and trajectory-aware error correction or time allocation\cite{zhou2024fast,wangdiffusion,chentrajectory,zhou2024simple,zhu2026realdyadic,wang2024pfdiff,zhou2026analyzing,min2025b}. Distillation-based methods compress multi-step sampling into few-step or one-step generation through progressive distillation\cite{salimans2022progressive}, consistency modeling\cite{luo2023latent}, or distribution matching\cite{yin2024one}. Feature caching methods instead keep the pretrained generator largely unchanged and reduce the cost of each denoiser evaluation by reusing or predicting internal representations\cite{ma2024deepcache,liu2025reusing}. Compared with the former two directions, feature caching is particularly deployment-friendly because it avoids retraining the full model and does not require redesigning the sampling dynamics.

Early training-free feature caching methods mainly followed a \emph{cache-then-reuse} paradigm, where intermediate features from nearby timesteps are directly reused to skip selected layers or modules\cite{ma2024deepcache,selvaraju2024fora,liu2025teaser,ma2025magcache}. Recent methods move toward \emph{cache-then-forecast}. TaylorSeer\cite{liu2025reusing} predicts future features with finite differences and Taylor expansion, and 
\begin{wrapfigure}{r}{0.60\linewidth}
\vspace{-1.2em}
\centering
\includegraphics[width=\linewidth]{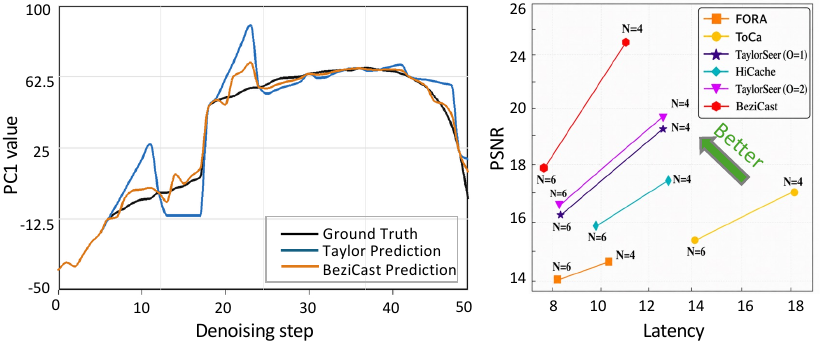}
\vspace{-0.5cm}
\caption{\textbf{Stable, Fast, and Accurate Diffusion Acceleration.} \textit{Left}: PC1 projections of output-proximal representation in denoising step. \textit{Right}: Speedup and PSNR comparison, where \(N\) is the cache interval. BeziCast achieves faster inference while better quality.}
\label{fig:motivation}
\vspace{-0.8cm}
\end{wrapfigure}
SpeCa\cite{liu2025speca} adds a lightweight verifier to accept or reject predicted features. Spectrum\cite{han2026spectrum} takes a global function-approximation view by fitting Chebyshev polynomial coefficients for final-layer denoiser features and deriving a horizon-independent error bound under its approximation setting. These methods show that denoiser representations have exploitable temporal regularity, but they also leave two design choices: which representation should be predicted, and which temporal basis should be used.

The first design choice is the \textbf{temporal basis}. Taylor-style predictors are simple, but finite-difference derivative estimates and local power-basis extrapolation can amplify small cache perturbations at longer offsets or higher orders\cite{liu2025reusing,li2023weakly,fenghicache}, as evidenced by the instability of the blue curve and its large deviation from the ground truth on the left of Fig.~\ref{fig:motivation}. This sensitivity is consistent with the numerical behavior of monomial and Vandermonde-type systems\cite{Shen_2025}, the non-monotone feature trajectories discussed by Feng et al.\cite{fenghicache}, and the large-horizon error growth observed by Han et al.\cite{han2026spectrum}. At the same time, diffusion sampling trajectories often exhibit structured low-dimensional geometry\cite{chen2025geometric}. We therefore seek a compact temporal parameterization that avoids explicit high-order finite differences while remaining analytically evaluable at cache steps.

The second design choice is the \textbf{forecasting target}. Existing caching methods often predict attention outputs, MLP outputs inside transformer blocks\cite{selvaraju2024fora,ma2025magcache,zhu2025learning,liu2025speca}. These local targets give fine-grained control, but their errors still pass through downstream denoiser blocks before reaching the final noise estimate and may accumulate across predicted modules. We instead predict an output-proximal denoising representation, which makes the path from representation error to denoising error shorter and more direct without assuming that this representation is intrinsically easier to forecast.

Based on these observations, we propose \textbf{BeziCast}, \textbf{a training-free \emph{Tikhonov-stabilized B\'ezier representation forecasting} framework, achieving faster, more stable, and more accurate acceleration than competing methods}, as shown on the right of Fig.~\ref{fig:motivation}.
BeziCast predicts an output-proximal denoising representation with a low-order Bernstein/B\'ezier trajectory whose control points are estimated by Tikhonov-stabilized fitting. The B\'ezier form provides analytic evaluation and a compact capacity-controlled function family\cite{farouki1987numerical,wu2026pear,patrikalakis2009shape}, while the Bernstein basis has classical stability properties on a fixed interval\cite{farouki1996optimal}, as shown by the orange curve on the left of Fig.~\ref{fig:motivation}, which is closer to the ground truth than Taylor-based extrapolation. Because cache queries are often short-range extrapolations, BeziCast does not use the convex-hull property as a guarantee; instead, it uses equivalent weights to detect signed or non-affine extrapolation and pulls high-risk predictions toward a simplex-projected conservative estimate. The pretrained denoiser and sampler remain unchanged.

Our contributions are as follows: 
(1) We formulate training-free cache-then-forecast acceleration as output-proximal representation forecasting, avoiding multiple local module targets.
(2) We use a low-order Bernstein/B\'ezier trajectory with Tikhonov-stabilized control-point fitting, replacing finite-difference Taylor extrapolation, decoupling trajectory capacity from cache-query density, and expressing each prediction as a regularized linear combination of cached representations.
(3) We introduce an equivalent-weight convex-hull guardrail that detects aggressive extrapolation through negative weight mass, deviation from unit sum, and the \(\ell_1\) norm, then pulls high-risk predictions toward a simplex-projected conservative estimate.

%% file: Contents/2related.tex
\section{Related Work}
\vspace{-0.25cm}

Diffusion acceleration methods mainly reduce either the number of denoising steps, the cost of each denoising step, or both.

\noindent\textbf{Sampler-based Acceleration.}
Sampler-based methods reduce the number of function evaluations by improving the reverse diffusion discretization. DDIM \cite{song2020denoising} enables deterministic non-Markovian sampling from DDPM \cite{ho2020denoising} models with fewer steps. DPM-Solver \cite{lu2022dpm} and DPM-Solver++ \cite{lu2025dpm} derive high-order ODE solvers for probability-flow sampling, UniPC \cite{zhao2023unipc} unifies predictor-corrector updates, and EDM \cite{karras2022elucidating} improves noise parameterization, preconditioning, and sampling schedules. Recent trajectory-aware methods further exploit path regularity, error correction, or time allocation, including AMED-Solver \cite{zhou2024fast}, Trajectory Regularity \cite{chentrajectory}, Simple and Fast Distillation \cite{zhou2024simple}, PAS \cite{wangdiffusion}, PFDiff \cite{wang2024pfdiff}, fast-sampling analysis \cite{zhou2026analyzing}, and B\'ezierFlow \cite{min2025b}. These methods mainly accelerate sampling by changing the solver, schedule, or trajectory correction strategy.

\noindent\textbf{Training-based Acceleration.}
Training-based methods compress a multi-step teacher into a faster generator through extra optimization. Progressive Distillation \cite{salimans2022progressive} halves sampling steps iteratively; Consistency Models \cite{song2023consistency} and Latent Consistency Models \cite{luo2023latent} enforce trajectory-consistent predictions for few-step generation; DMD \cite{yin2024one} and DMD2 \cite{yin2024improved} match the teacher distribution rather than the full teacher trajectory; and InstaFlow \cite{liu2023instaflow} uses rectified-flow-style training to straighten text-to-image generation paths. These approaches can be effective, but they usually require additional training, teacher models, or specialized objectives.

% \noindent\textbf{Feature-caching-based Acceleration.}
% Feature-caching methods reduce inference cost by exploiting temporal redundancy in denoiser activations while keeping the pretrained denoiser largely unchanged. Early cache-then-reuse methods directly reuse intermediate features or residuals: DeepCache \cite{ma2024deepcache} reuses high-level U-Net features, FORA \cite{selvaraju2024fora} extends fast-forward caching to diffusion transformers, and Reuse and Diffuse \cite{gu2023reuse} plus MagCache \cite{ma2025magcache} study feature or residual reuse for video diffusion. Recent cache-then-forecast methods model feature evolution instead of copying old features: TaylorSeer \cite{liu2025reusing} uses finite differences and Taylor expansion, FreqCa \cite{liu2025freqca} adds frequency-aware caching, Forecast-then-Calibrate \cite{zheng2026forecast} introduces feature calibration, HiCache \cite{fenghicache} uses higher-order interpolation, and Spectrum \cite{han2026spectrum} fits global spectral approximations. Adaptive variants decide when or where to recompute, including DiCache \cite{bu2025dicache} for dynamic cache scheduling, SpeCa \cite{liu2025speca} for feature verification, and token-level cache updates \cite{zheng2025compute}. BeziCast follows the training-free cache-then-forecast line, but focuses on output-proximal representation forecasting with a Tikhonov-stabilized Bernstein/B\'ezier trajectory and an equivalent-weight guardrail for extrapolation risk.

\noindent\textbf{Feature-caching-based Acceleration.}
Feature-caching methods\cite{liu2025reusing,yang2024showmaker,selvaraju2024fora,yang2025gesturehydra} reduce inference cost by exploiting temporal redundancy in denoiser activations while keeping the pretrained denoiser largely unchanged. Early cache-then-reuse methods directly reuse intermediate features or residuals: DeepCache~\cite{ma2024deepcache} reuses high-level U-Net features, FORA~\cite{selvaraju2024fora} extends fast-forward caching to diffusion transformers, ToCa~\cite{zou2024accelerating} introduces token-wise feature caching, and Reuse and Diffuse~\cite{gu2023reuse} plus MagCache~\cite{ma2025magcache} study feature or residual reuse for video diffusion. Recent cache-then-forecast methods model feature evolution instead of copying old features: TaylorSeer~\cite{liu2025reusing} uses finite differences and Taylor expansion, FreqCa~\cite{liu2025freqca} adds frequency-aware caching, Forecast-then-Calibrate~\cite{zheng2026forecast} introduces feature calibration, HiCache~\cite{fenghicache} replaces Taylor-style monomial forecasting with Hermite-polynomial-based prediction, and Spectrum~\cite{han2026spectrum} fits global spectral approximations. Adaptive variants such as TeaCache~\cite{liu2025timestep}, DiCache~\cite{bu2025dicache}, SpeCa~\cite{liu2025speca}, and token-level cache updates~\cite{zheng2025compute} further show that diffusion features exhibit strong temporal and structural redundancy. BeziCast follows the training-free cache-then-forecast line, but focuses on output-proximal representation forecasting with a Tikhonov-stabilized B\'ezier trajectory and an equivalent-weight guardrail for extrapolation risk.

%% file: Contents/3method.tex
\section{Method}
\label{sec:method}
\vspace{-0.25cm}

% \begin{figure}[t]
%     \centering
%     \includegraphics[width=1.\linewidth]{imgs/Pipeline_1_cropped.pdf}
%     \vspace{-1.5em}
%     \caption{
%     Overview of BeziCast.
%     }
%     \label{fig:pipeline}
%     \vspace{-1.0cm}
% \end{figure}

\begin{figure}[t]
    \centering
    \includegraphics[width=0.9\linewidth]{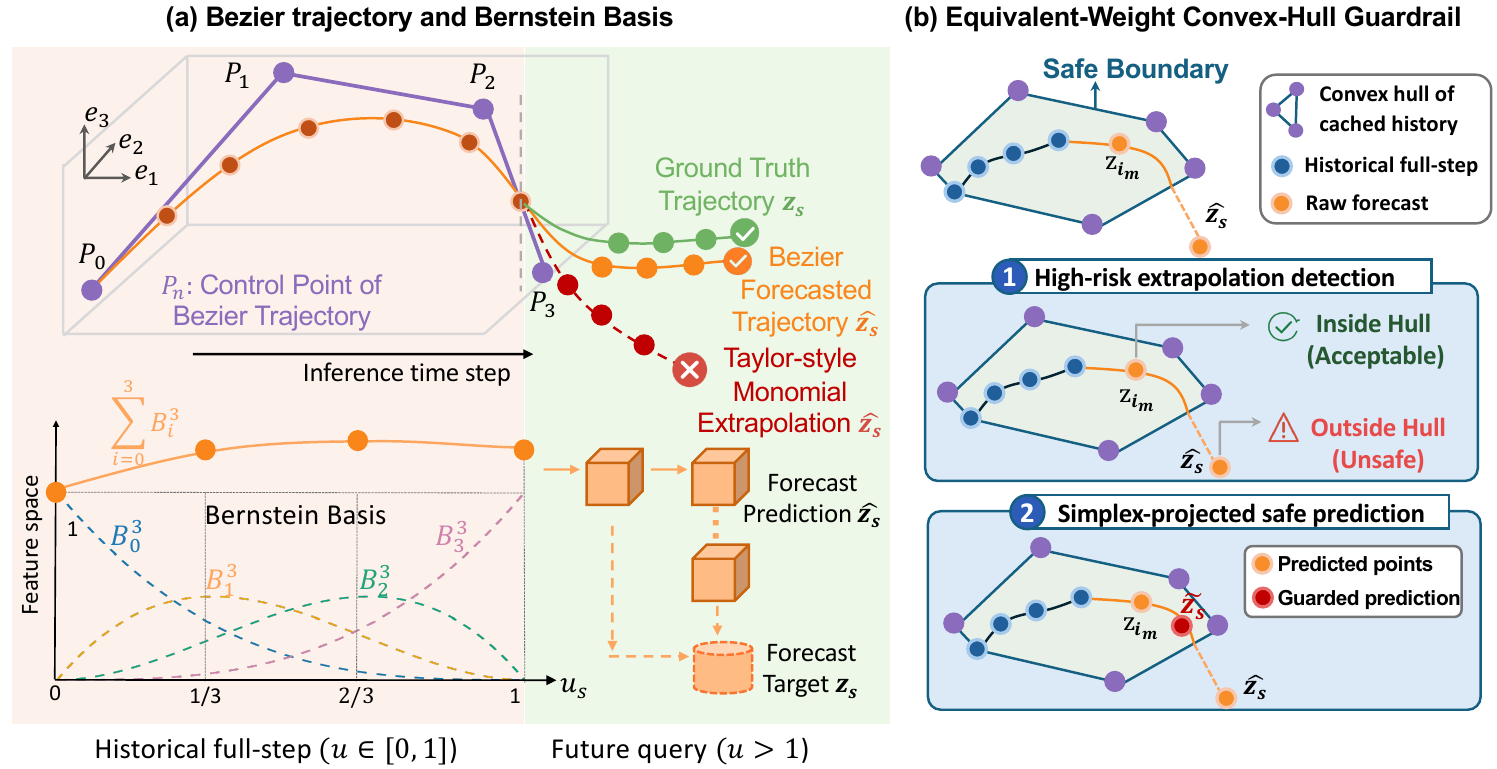}
    \vspace{-0.6em}
    \caption{
    % \textbf{Overview of BeziCast.}
    % (a) BeziCast forecasts cached output-proximal representations with a low-order B\'ezier trajectory, avoiding unstable monomial-basis extrapolation.
    % (b) An equivalent-weight guardrail detects risky extrapolations and pulls them toward a simplex-projected estimate.
    \textbf{Overview of BeziCast.}
(a) A local cubic Bernstein/B\'ezier trajectory is fitted from cached full-step representations and evaluated at cache queries \(u_s > 1\), providing a low-order alternative to Taylor-style extrapolation.
(b) BeziCast detects risky extrapolation via equivalent history weights, constructs a simplex-projected prediction inside the cached-representation convex hull, and softly blends it with the raw forecast to obtain the guarded prediction.
}
    \label{fig:pipeline}
    \vspace{-0.5cm}
\end{figure}

BeziCast is a training-free cache-then-forecast method for diffusion
sampling. We recast feature prediction as output-proximal representation
trajectory approximation: from a short history of denoising representations
near the output, the method estimates a low-order B\'ezier trajectory
and evaluates it at skipped cache steps. As shown in Fig.~\ref{fig:pipeline},
BeziCast runs the original denoiser at selected full steps, caches
output-proximal representations, predicts them at cache steps with
Tikhonov-stabilized Bernstein fitting, and applies the guardrail when needed.
It does not train additional models,
modify the sampler, or change the pretrained denoiser. This section first
defines the cache-then-forecast setting, then gives the predictor and guardrail.
The full inference algorithm is provided in
Appendix~\ref{app:online_algorithm}.
\vspace{-0.25cm}

\subsection{Cache-Then-Forecast Formulation}
\label{sec:unified_framework}

\noindent\textbf{Definition 1 (diffusion sampling).}
Let $\mathcal{T}=\{t_0,t_1,\ldots,t_{T-1}\}$ be the sampling time grid, where
\(s\in\{0,\ldots,T-1\}\) is the solver-step index. Given latent state \(x_s\),
condition \(c\), and diffusion time \(t_s\), the pretrained denoiser predicts
\(\epsilon_s=\epsilon_\theta(x_s,t_s,c)\).
The sampler updates the latent by
\begin{equation}
    x_{s+1}=\Psi_s(x_s,\epsilon_s),
    \label{eq:sampler_update}
\end{equation}
where \(\Psi_s\) denotes one solver step. We seek to reduce expensive denoiser
computation while keeping both the pretrained denoiser \(\epsilon_\theta\) and
the sampler update \(\Psi_s\) unchanged.

% A key design choice in cache-then-forecast acceleration is the representation
% to cache and predict. Existing methods often forecast attention outputs, MLP
% outputs, or residual branches inside transformer
% blocks~\cite{selvaraju2024fora,ma2025magcache,liu2025speca}. These local
% targets offer fine-grained control, but their errors pass through downstream
% nonlinear blocks before reaching the final noise estimate, and multiple local
% predictions can accumulate within one denoising step. We therefore forecast a
% single output-proximal representation immediately before the final output map.
% This target placement also follows the final-block forecasting setting studied
% by Han et al.~\cite{han2026spectrum}.

A key design choice in cache-then-forecast acceleration is the representation to cache and predict. Prior works typically forecast local features such as attention, MLP, or residual outputs within transformer blocks~\cite{selvaraju2024fora,ma2025magcache,wu2025swift4d,liu2025speca}, but errors in these intermediate targets are propagated through subsequent nonlinear layers and can accumulate if multiple predictions occur per denoising step. To ensure errors are directly reflected in the final noise estimate and to simplify control, we forecast a single output-proximal representation immediately before the final output map, aligning with the final-block forecasting setting of Spectrum~\cite{han2026spectrum}.

\begin{figure}[t]
    \centering
    \includegraphics[width=0.85\linewidth]{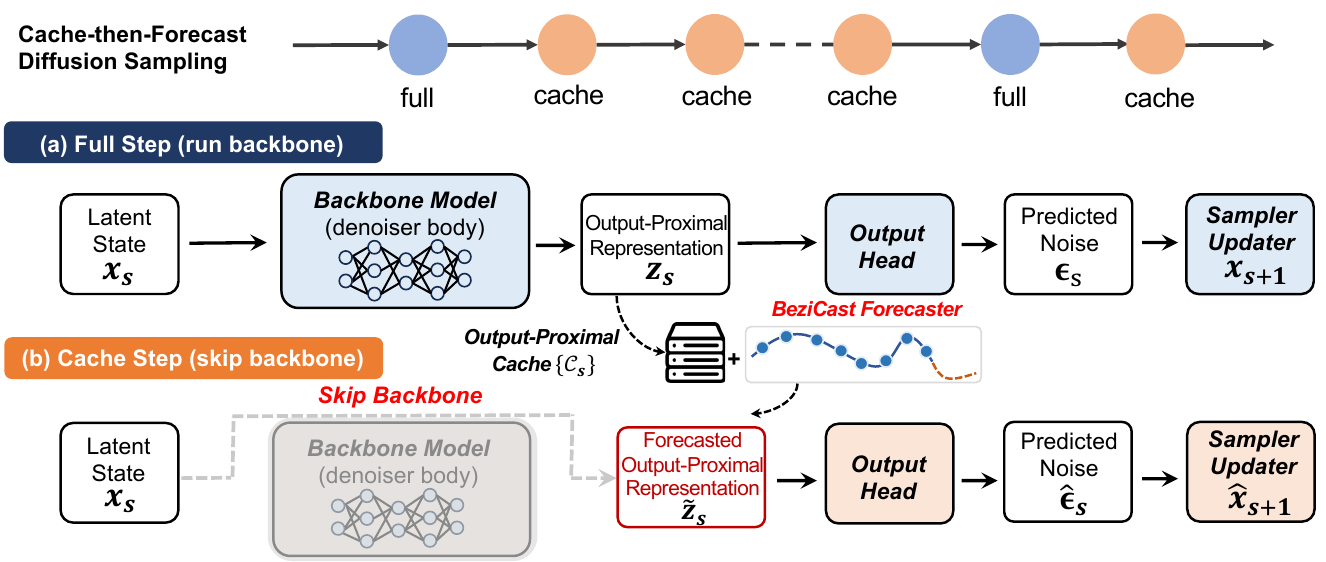}
    \vspace{-0.05in}
    \caption{
BeziCast sampling pipeline. In the cache-then-forecast paradigm, selected full steps run the denoiser backbone and cache output-proximal representations, while intermediate cache steps skip the backbone and forecast the missing representation from history. BeziCast forecasts \(\tilde z_s\) with Tikhonov-stabilized B\'ezier fitting and then uses the unchanged output head and sampler update.
    % BeziCast sampling pipeline. Full steps cache the output-proximal representation \(z_s\) from the denoiser backbone. Cache steps skip the backbone and forecast \(\tilde z_s\) with Tikhonov-stabilized B\'ezier fitting.
    }
    \label{fig:sampling_pipeline}
    \vspace{-0.4cm}
\end{figure}

\noindent\textbf{Definition 2 (output-proximal representation).}
\label{sec:output_proximal}
We decompose the pretrained denoiser into a backbone and a final output map:
\begin{equation}
    z_s
    =
    \mathcal{F}_\theta(x_s,t_s,c),
    \qquad
    \epsilon_\theta(x_s,t_s,c)
    =
    \mathcal{O}_\theta(z_s,t_s,c).
    \label{eq:denoiser_decomposition}
\end{equation}
Here \(z_s\in\mathbb{R}^{D}\) is the flattened output-proximal denoising
representation, and \(D\) absorbs token and channel dimensions. The map
\(\mathcal{F}_\theta\) denotes the expensive backbone computation up to this
representation, while \(\mathcal{O}_\theta\) denotes the lightweight output map
that converts \(z_s\) into the denoising prediction. The output map is always
executed, both at full steps and cache steps.

\noindent\textbf{Proposition 1 (output-proximal error path).}
Assume the output map \(\mathcal{O}_\theta\) is locally Lipschitz with respect
to its representation input. If \(\hat z_s=z_s+e_s^z\) is an
output-proximal prediction, then
\begin{equation}
    \|\mathcal{O}_\theta(\hat z_s,t_s,c)-\mathcal{O}_\theta(z_s,t_s,c)\|
    \le
    L_{\mathcal{O}}\|e_s^z\|,
    \label{eq:output_proximal_error}
\end{equation}
where \(L_{\mathcal{O}}\) is a local Lipschitz constant of the output map.

Proposition 1 supports our choice of forecasting a single output-proximal representation immediately before the output mapping through the output-map Lipschitz relation. Appendix~\ref{app:output_proximal_analysis} provides the corresponding local Jacobian/Lipschitz analysis and proofs, including the derivations for intermediate-layer and sampler-level error propagation.

\noindent\textbf{Definition 3 (training-free cache forecaster).}
Let
\begin{equation}
    \mathcal{S}_{\mathrm{full}}\subseteq\{0,\ldots,T-1\},
    \qquad
    \mathcal{S}_{\mathrm{cache}}
    =
    \{0,\ldots,T-1\}\setminus\mathcal{S}_{\mathrm{full}}
\end{equation}
denote full-step and cache-step indices. At a full step
\(s\in\mathcal{S}_{\mathrm{full}}\), the method runs the original denoiser
backbone, obtains \(z_s\) and \(\epsilon_s\) by
Eq.~\eqref{eq:denoiser_decomposition}, and stores \(z_s\). At a cache step
\(s\in\mathcal{S}_{\mathrm{cache}}\), it skips the backbone, predicts
a raw representation \(\hat z_s\) from cached representations, optionally
forms the guarded representation \(\tilde z_s\), and computes
\(\hat\epsilon_s=\mathcal{O}_\theta(\tilde z_s,t_s,c)\) with the unchanged
output map. Without guardrail activation, \(\tilde z_s=\hat z_s\). The sampler
then uses \(\hat\epsilon_s\) in Eq.~\eqref{eq:sampler_update}.

Let $\mathcal{C}_s = \{(\tau_{i_m},z_{i_m})\}_{m=1}^{K}$
be the local cache at step \(s\), where \(i_1<\cdots<i_K<s\) are the most
recent \(K\) full steps and \(\tau_i\) is a monotone scalar coordinate
associated with \(t_i\). A cache forecaster has the form
\begin{equation}
    \hat z_s
    =
    f(\mathcal{C}_s,\tau_s),
    \qquad
    s\in\mathcal{S}_{\mathrm{cache}}.
    \label{eq:generic_forecaster}
\end{equation}
Different cache-then-forecast methods correspond to different choices of
\(f\). Direct reuse sets \(f(\mathcal{C}_s,\tau_s)=z_{i_K}\). Taylor-style
forecasting estimates finite differences from nearby cached representations
and extrapolates them to \(\tau_s\). BeziCast instead uses the
regularized Bernstein/B\'ezier forecaster defined below.

In our cubic B\'ezier setting, we use a warm-up plus fixed-period schedule
(see Appendix~\ref{app:warmup_fixed_schedule}). The
first \(W\ge4\) solver steps are computed fully to initialize the
representation history. After warm-up, one full denoiser evaluation is
performed every \(R\) steps, and each intermediate cache step is forecast from
at most \(K_{\max}=6\) recent full-step representations.

\subsection{Tikhonov-Stabilized B\'ezier Forecaster}
\label{sec:bezier_design}
\label{sec:basis_motivation}

The temporal basis determines how cached denoising representations are
predicted at skipped timesteps. Taylor-style predictors estimate finite
differences and extrapolate them in a Newton-form power basis. BeziCast
uses a low-order Bernstein/B\'ezier trajectory with Tikhonov-stabilized fitting.
The forecaster is specified below by local time normalization, a Bernstein
basis, and a Tikhonov control-point estimate.
Appendix~\ref{app:bernstein_design_motivation} gives the design motivation and perturbation discussion.

Since future full-step representations are unavailable during online inference, cache-step prediction is typically performed by extrapolating beyond the recent history window. The standard B\'ezier convex-hull property only holds inside the interpolation interval, so we control extrapolation by clipping the query coordinate with \(u_{\max}\) and by applying the equivalent-weight guardrail in Section~\ref{sec:convex_hull_guardrail}.

\noindent\textbf{Definition 4 (local time normalization).}
For a cache step \(s\), let
\(\mathcal{I}_s=\{i_1,\ldots,i_K\}\), with \(i_1<\cdots<i_K<s\), be the
indices of the most recent \(K\) full steps. Historical coordinates are
normalized by
\begin{equation}
    u_m
    =
    \frac{\tau_{i_m}-\tau_{i_1}}
    {\tau_{i_K}-\tau_{i_1}+\varepsilon_\tau},
    \qquad
    m=1,\ldots,K,
    \label{eq:history_normalization}
\end{equation}
where \(\varepsilon_\tau>0\) is a small numerical constant. The cache-step
query coordinate is normalized with the same local window and clipped by
\(u_{\max}\):
\begin{equation}
    u_s
    =
    \frac{\tau_s-\tau_{i_1}}
    {\tau_{i_K}-\tau_{i_1}+\varepsilon_\tau},
    \qquad
    u_s\leftarrow \min(u_s,u_{\max}),
    \label{eq:query_normalization}
\end{equation}
The cap \(u_{\max}\) is a fixed hyperparameter for the extrapolation coordinate.

\noindent\textbf{Definition 5 (Bernstein/B\'ezier trajectory).}
For \(n\) control representations
\(P_0,\ldots,P_{n-1}\in\mathbb{R}^{D}\), define
\begin{equation}
    \hat z(u)
    =
    \sum_{j=0}^{n-1}
    b_{j,n-1}(u)P_j,
    \qquad
    b_{j,n-1}(u)
    =
    \binom{n-1}{j}u^j(1-u)^{n-1-j}.
    \label{eq:bezier_trajectory}
\end{equation}
The control representations parameterize a local trajectory; they are fitted
parameters and need not coincide with any cached representation.

\noindent\textbf{Definition 6 (Tikhonov control-point fitting).}
Let $b(u) = [b_{0,n-1}(u),\ldots,b_{n-1,n-1}(u)]^\top\in\mathbb{R}^{n}$,
and construct the Bernstein design matrix \(A\in\mathbb{R}^{K\times n}\) by
\(A_{m,j}=b_{j,n-1}(u_m)\).
Stack the cached output-proximal representations as
\begin{equation}
    Z=
    \begin{bmatrix}
        z_{i_1} & z_{i_2} & \cdots & z_{i_K}
    \end{bmatrix}^\top
    \in\mathbb{R}^{K\times D}.
\end{equation}
We estimate the control matrix \(P\in\mathbb{R}^{n\times D}\) (composed of control points) by
\begin{equation}
    P_\lambda^\star
    =
    \arg\min_P
    \|AP-Z\|_F^2+
    \lambda\|\Gamma P\|_F^2,
    \label{eq:tikhonov_objective}
\end{equation}
where \(\lambda>0\) is the regularization strength and \(\Gamma\) is the
Tikhonov matrix. We set \(\Gamma=I\), yielding the
ridge-stabilized solution
\begin{equation}
    P_\lambda^\star
    =
    (A^\top A+\lambda I)^{-1}A^\top Z.
    \label{eq:ridge_solution}
\end{equation}

\noindent\textbf{Proposition 2 (equivalent linear smoother).}
With \(\Gamma=I\), the cache-step prediction can be written as
\begin{equation}
    \hat z_s
    =
    b(u_s)^\top(A^\top A+\lambda I)^{-1}A^\top Z.
    \label{eq:ridge_bezier_prediction}
\end{equation}
Equivalently,
\begin{equation}
    \hat z_s
    =
    \omega_\lambda(u_s)^\top Z,
    \qquad
    \omega_\lambda(u_s)^\top
    =
    b(u_s)^\top(A^\top A+\lambda I)^{-1}A^\top
    \in\mathbb{R}^{1\times K}.
    \label{eq:equivalent_weight}
\end{equation}

\noindent\emph{Proof.}
Substituting Eq.~\eqref{eq:ridge_solution} into
Eq.~\eqref{eq:bezier_trajectory} gives Eq.~\eqref{eq:ridge_bezier_prediction}.
Collecting the coefficients applied to \(Z\) gives Eq.~\eqref{eq:equivalent_weight}.
Eq.~\eqref{eq:equivalent_weight} rewrites the prediction as a weighted sum of
cached representations. The equivalent weights depend on the local coordinates,
the Bernstein basis, and \(\lambda\).

\noindent\textbf{Proposition 3 (regularization-controlled sensitivity).}
For the ridge-stabilized predictor in Eq.~\eqref{eq:ridge_bezier_prediction},
any cache perturbation \(\delta Z\) satisfies
\begin{equation}
    \|\delta \hat z_s\|_F
    \le
    \frac{\|b(u_s)\|_2}{2\sqrt{\lambda}}
    \|\delta Z\|_F.
    \label{eq:ridge_bezier_stability}
\end{equation}
On the clipped query range \(u_s\in[0,u_{\max}]\), the factor
\(\|b(u_s)\|_2\) is bounded by
\(\sup_{u\in[0,u_{\max}]}\|b(u)\|_2\). The proof and the comparison with
ordinary least squares are given in Appendix~\ref{app:bernstein_stability}.

\subsection{Equivalent-Weight Convex-Hull Guardrail}
\label{sec:convex_hull_guardrail}

The equivalent weights in Eq.~\eqref{eq:equivalent_weight} are also used to
define a coefficient-space score. When \(u_s>1\), the induced historical
weights may become signed, deviate from unit sum, or have a large \(\ell_1\)
norm. The following definitions specify the score and the simplex-projected
fallback.

\noindent\textbf{Definition 7 (equivalent-weight risk).}
Let \(w=\omega_\lambda(u_s)\in\mathbb{R}^{K}\). Define
\begin{equation}
    r(u_s)
    =
    |\mathbf{1}^\top w-1|
    +
    \bigl(\|w\|_1-|\mathbf{1}^\top w|\bigr).
    \label{eq:risk_score}
\end{equation}
The first term measures deviation from an affine combination. The second term
measures excess signed mass.

\noindent\textbf{Proposition 4 (risk-score interpretation).}
The score \(r(u_s)\) is nonnegative. Moreover, \(r(u_s)=0\) if and only if
\(w\) lies in the probability simplex
\begin{equation}
    \Delta
    =
    \{\alpha\in\mathbb{R}^{K}:\alpha_m\ge 0,\ \mathbf{1}^\top\alpha=1\}.
\end{equation}
If \(\mathbf{1}^\top w=1\), then
\begin{equation}
    r(u_s)
    =
    \|w\|_1-1
    =
    2\sum_{m=1}^{K}\max(-w_m,0).
\end{equation}

\noindent\emph{Proof.}
The inequality \(\|w\|_1\ge |\mathbf{1}^\top w|\) gives nonnegativity. If
\(r(u_s)=0\), then \(\mathbf{1}^\top w=1\) and \(\|w\|_1=1\), which implies
\(w_m\ge 0\) for all \(m\). Thus \(w\in\Delta\). The converse is immediate.
When \(\mathbf{1}^\top w=1\), decomposing \(w\) into positive and negative
parts gives \(\|w\|_1-1=2\sum_m\max(-w_m,0)\).

\noindent\textbf{Definition 8 (guarded prediction).}
We project the equivalent weights onto the simplex:
\begin{equation}
    w_{\mathrm{safe}}(u_s)
    =
    \Pi_\Delta(w),
    \label{eq:simplex_projection}
\end{equation}
where \(\Pi_\Delta\) denotes Euclidean projection. The conservative
representation is
\begin{equation}
    z_s^{\mathrm{safe}}
    =
    w_{\mathrm{safe}}(u_s)^\top Z.
    \label{eq:safe_prediction}
\end{equation}
The final cache-step representation is
\begin{equation}
    \tilde z_s
    =
    (1-\beta_s)\hat z_s+
    \beta_s z_s^{\mathrm{safe}},
    \label{eq:guardrail_blending}
\end{equation}
with
\begin{equation}
    \beta_s
    =
    \beta_{\max}
    \cdot
    \operatorname{clip}
    \left(
    \frac{r(u_s)-\rho}{\gamma},
    0,
    1
    \right).
    \label{eq:beta_schedule}
\end{equation}
Here \(\rho\) is the activation threshold, \(\gamma\) controls transition
smoothness, and \(\beta_{\max}\le 1\) is the maximum blending strength.

\noindent\textbf{Proposition 5 (convex-hull fallback).}
The representation \(z_s^{\mathrm{safe}}\) lies in the convex hull of the cached
representations \(\{z_{i_m}\}_{m=1}^{K}\). If \(\beta_s=1\), the final
prediction \(\tilde z_s\) equals this convex-hull fallback. If
\(0<\beta_s<1\), the final prediction is pulled toward the fallback but is not
forced to lie inside the historical convex hull.

\noindent\emph{Proof.}
By construction, \(w_{\mathrm{safe}}\in\Delta\). Therefore
\(z_s^{\mathrm{safe}}=w_{\mathrm{safe}}^\top Z\) is a convex combination of the
cached representations. The statements about \(\tilde z_s\) follow directly
from Eq.~\eqref{eq:guardrail_blending}.
% \hfill\(\square\)

% \subsection{Complexity}
% Let \(D\) be the flattened representation dimension, \(K_s\le K_{\max}\) the
% available history length at cache step \(s\), and \(n\) the number of
% B\'ezier control points. At one cache step, building the design matrix costs
% \(O(K_sn)\). Solving the \(n\times n\) ridge system costs
% \(O(n^3)\), and this part can be precomputed when the cache schedule and
% hyperparameters are fixed. The main high-dimensional operation is the weighted
% sum \(\omega_\lambda(u_s)^\top Z\), which costs \(O(K_sD)\). The guardrail adds a
% simplex projection in \(K_s\) dimensions, costing \(O(K_s\log K_s)\) with a
% sorting-based algorithm, and one additional weighted sum. The memory overhead
% is \(O(K_{\max}D)\) for cached output-proximal representations.
% Algorithm~\ref{alg:BeziCast} does not store the control
% representations persistently; prediction uses equivalent historical weights.

%% file: Contents/4exp.tex
\section{Experiments}
\label{sec:experiments}
\vspace{-0.25cm}

We conduct comprehensive experiments on two representative generative tasks using SOTA models, covering text-to-image generation with FLUX.1-dev\cite{esser2024scaling} and text-to-video generation with HunyuanVideo\cite{kong2024hunyuanvideo}. Each task follows standard evaluation protocols and widely adopted benchmarks, enabling fair and reproducible comparisons with existing methods. We present additional BeziCast configurations and component-level ablations in separate analysis subsections.

\begin{table}[t]
    \centering
    \caption{
    Benchmark results on text-to-image generation with FLUX.1 and text-to-video generation with HunyuanVideo. 
    We use 50 steps as the reference ($^{\dagger}$). 
    Our \textit{BeziCast} achieves higher speedup while maintaining better sample quality across two speedup scenarios consistently.
    }
    \label{tab:combined_t2i_t2v_comparison}
    \scriptsize
    \setlength{\tabcolsep}{2.0pt}
    \renewcommand{\arraystretch}{1.0}
    \resizebox{\linewidth}{!}{
    \begin{tabular}{lccccccc@{\hspace{0.8em}}cccccc}
    \toprule
    \multirow{3}{*}{\textbf{Method}}
    & \multicolumn{7}{c}{\textbf{FLUX.1 [dev] / DrawBench}}
    & \multicolumn{6}{c}{\textbf{HunyuanVideo / VBench}} \\
    \cmidrule(lr){2-8}
    \cmidrule(lr){9-14}
    & \multicolumn{2}{c}{\textbf{Acceleration}}
    & \multicolumn{3}{c}{\textbf{Quality}}
    & \multirow{2}{*}{\makecell[c]{\textbf{Image}\\\textbf{Reward}$\uparrow$}}
    & \multirow{2}{*}{\textbf{CLIP}$\uparrow$}
    & \multicolumn{2}{c}{\textbf{Acceleration}}
    & \multicolumn{3}{c}{\textbf{Quality}}
    & \multirow{2}{*}{\makecell[c]{\textbf{VBench}\\\textbf{Quality}$\uparrow$}} \\
    \cmidrule(lr){2-3}
    \cmidrule(lr){4-6}
    \cmidrule(lr){9-10}
    \cmidrule(lr){11-13}
    & \textbf{Latency(s)$\downarrow$}
    & \textbf{Speedup$\uparrow$}
    & \textbf{PSNR$\uparrow$}
    & \textbf{SSIM$\uparrow$}
    & \textbf{LPIPS$\downarrow$}
    &
    &
    &
    \textbf{Latency(s)$\downarrow$}
    & \textbf{Speedup$\uparrow$}
    & \textbf{PSNR$\uparrow$}
    & \textbf{SSIM$\uparrow$}
    & \textbf{LPIPS$\downarrow$}
    & \\
    \midrule
    50 steps$^{\dagger}$
    & 37.681 & 1.00$\times$ & $-$ & $-$ & $-$ & 1.052 & 27.01
    & 489.18 & 1.00$\times$ & $-$ & $-$ & $-$ & 84.03 \\
    25 steps
    & 19.583 & 1.92$\times$ & 17.350 & 0.679 & 0.228 & 0.998 & 27.23
    & 252.40 & 1.94$\times$ & 21.779 & 0.885 & 0.243 & 83.60 \\
    15 steps
    & 12.379 & 3.04$\times$ & 15.191 & 0.602 & 0.329 & 0.953 & 27.38
    & 157.06 & 3.12$\times$ & 18.894 & 0.796 & 0.382 & 83.47 \\
    \midrule
    FORA ($\mathcal{N}=4$)
    & 11.208 & 3.36$\times$ & 14.702 & 0.584 & 0.356 & 0.885 & 27.20
    & 153.67 & 3.18$\times$ & 17.892 & 0.740 & 0.421 & 83.21 \\
    ToCa ($\mathcal{N}=4$, $\mathcal{R}=0.9$)
    & 17.995 & 2.09$\times$ & 16.560 & 0.599 & 0.251 & 0.974 & 27.01
    & 336.76 & 1.46$\times$ & 17.837 & 0.736 & 0.425 & 83.21 \\
    TeaCache ($\delta=0.4/0.2$)
    & 20.349 & 1.86$\times$ & 17.100 & 0.632 & 0.228 & 1.002 & \textbf{27.32}
    & 178.64 & 2.74$\times$ & 19.869 & 0.796 & 0.371 & 83.46 \\
    TaylorSeer ($\mathcal{N}=4$, $O=1$)
    & 12.362 & 3.05$\times$ & 19.158 & 0.720 & 0.149 & \textbf{1.041} & 27.00
    & 152.68 & 3.20$\times$ & 20.629 & 0.877 & 0.281 & 82.81 \\
    TaylorSeer ($\mathcal{N}=4$, $O=2$)
    & 12.531 & 3.00$\times$ & 19.486 & 0.725 & 0.144 & 1.035 & 26.99
    & 152.87 & 3.12$\times$ & 20.475 & 0.869 & 0.290 & 82.25 \\
    HiCache ($\mathcal{N}=4$, $O=2$, $\sigma=0.5$)
    & 12.654 & 2.99$\times$ & 17.511 & 0.666 & 0.206 & 1.010 & 26.89
    & 149.63 & 3.27$\times$ & 20.866 & 0.874 & 0.274 & 83.06 \\
    \rowcolor{gray!15}
    \textit{\textbf{BeziCast}} ($\mathcal{N}=4/5$, $n=4$, $\lambda=0.5$)
    & \textbf{10.855} & \textbf{3.47}$\times$ & \textbf{22.382} & \textbf{0.782} & \textbf{0.090} & 1.035 & 27.17
    & \textbf{135.82} & \textbf{3.60}$\times$ & \textbf{24.783} & \textbf{0.916} & \textbf{0.197} & \textbf{83.48} \\
    \midrule
    FORA ($\mathcal{N}=6$)
    & 8.321 & 4.52$\times$ & 14.056 & 0.551 & 0.408 & 0.798 & 26.96
    & 116.10 & 4.21$\times$ & 16.883 & 0.652 & 0.504 & 82.55 \\
    ToCa ($\mathcal{N}=6$, $\mathcal{R}=0.9$)
    & 14.342 & 2.64$\times$ & 15.037 & 0.579 & 0.276 & 0.965 & 27.07
    & 230.32 & 2.13$\times$ & 16.861 & 0.648 & 0.504 & 82.39 \\
    TeaCache ($\delta=0.8/0.25$)
    & 12.833 & 2.94$\times$ & 15.106 & 0.621 & 0.289 & 0.960 & 27.03
    & 157.27 & 3.11$\times$ & 18.768 & 0.721 & 0.423 & 83.01 \\
    TaylorSeer ($\mathcal{N}=6$, $O=1$)
    & 8.793 & 4.31$\times$ & 15.617 & 0.621 & 0.286 & 0.940 & 26.41
    & 125.36 & 3.90$\times$ & 17.561 & 0.784 & 0.419 & 82.80 \\
    TaylorSeer ($\mathcal{N}=6$, $O=2$)
    & 8.862 & 4.27$\times$ & 15.872 & 0.598 & 0.281 & \textbf{0.988} & 26.52
    & 125.54 & 3.90$\times$ & 18.041 & 0.800 & 0.400 & 82.52 \\
    HiCache ($\mathcal{N}=6$, $O=2$, $\sigma=0.5$)
    & 9.663 & 3.90$\times$ & 15.618 & 0.585 & 0.296 & 0.971 & 26.54
    & 123.47 & 3.96$\times$ & 19.143 & 0.827 & 0.361 & 82.29 \\
    \rowcolor{gray!15}
    \textit{\textbf{BeziCast}} ($\mathcal{N}=6$, $n=4$, $\lambda=0.5$)
    & \textbf{7.881} & \textbf{4.79}$\times$ & \textbf{17.204} & \textbf{0.625} & \textbf{0.240} & 0.987 & \textbf{27.09}
    & \textbf{118.62} & \textbf{4.11}$\times$ & \textbf{23.631} & \textbf{0.904} & \textbf{0.217} & \textbf{83.21} \\
    \bottomrule
    \end{tabular}}
    \vspace{-0.5cm}
    \begin{minipage}{0.98\linewidth}
    % \small
    % $^{\dagger}$ denotes the 50-step full-sampling reference. 
    \end{minipage}
\end{table}

\subsection{Experimental Setup}

\paragraph{Models and Tasks.}
For text-to-image generation, we evaluate BeziCast on FLUX.1-dev~\cite{esser2024scaling} using a 50-step reference schedule and the 200-prompt DrawBench suite~\cite{saharia2022photorealistic}, generating images at \(1024 \times 1024\). For text-to-video generation, we use HunyuanVideo~\cite{kong2024hunyuanvideo} and follow the official VBench~\cite{huang2024vbench} prompt set and evaluation protocol. 
%For class-conditional image generation, we evaluate on ImageNet with DiT-XL/2~\cite{peebles2023scalable}.
\vspace{-0.4cm}

\paragraph{Baselines and Metrics.}
We compare BeziCast with representative training-free acceleration methods, including \(\kappa\)-step sampling, cache-reuse methods FORA~\cite{selvaraju2024fora}, ToCa~\cite{zou2024accelerating}, and TeaCache~\cite{liu2025timestep}, as well as forecasting-based methods TaylorSeer~\cite{liu2025reusing} and HiCache~\cite{fenghicache}. We use \(\mathcal{N}\) for the cache interval, \(n\) for the number of B\'ezier control representations, and \(\lambda\) for Tikhonov regularization. We report sampling latency (per-image/video) and speedup relative to the original 50-step sampler, together with task-specific quality metrics: PSNR, SSIM, LPIPS, ImageReward, and CLIP Score~\cite{hessel2021clipscore} for text-to-image generation; PSNR, SSIM, LPIPS, and VBench Quality Score for text-to-video generation.
\vspace{-0.4cm}

\begin{figure}[t]
    \centering
    \includegraphics[width=0.95\linewidth]{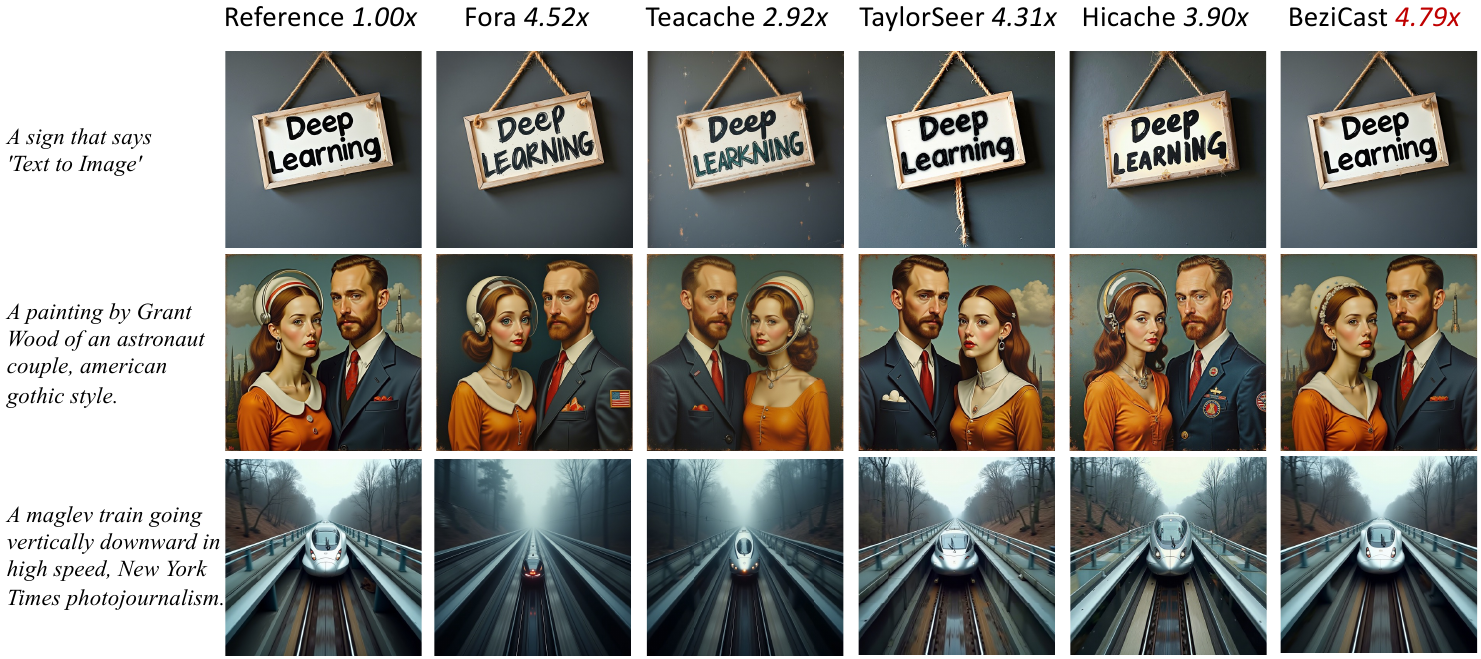}
    \caption{
    Qualitative comparison on text-to-image generation using FLUX.1. BeziCast delivers the highest speedup while preserving prompt alignment, text readability, and visual fidelity.
    }
    \label{fig:flux_qualitative}
    \vspace{-0.5cm}
\end{figure}

% \paragraph{Metrics.}
% We report both efficiency and quality metrics for all evaluated methods. Efficiency is measured by sampling latency and empirical speedup relative to the original 50-step sampler, with latency reported in seconds per image for text-to-image and class-conditional image generation, and seconds per video for text-to-video generation. For text-to-image generation, we compare accelerated samples against the 50-step reference outputs using PSNR, SSIM, and LPIPS, and additionally report ImageReward and CLIP Score~\cite{hessel2021clipscore} to measure preference alignment and text-image consistency. For text-to-video generation, we report PSNR, SSIM, LPIPS, and the VBench Quality Score (\%) following the official VBench protocol. For class-conditional image generation on ImageNet, we evaluate distribution-level quality using FID, sFID, and Inception Score.

\subsection{Main Comparisons.}
\label{sec:main_comparisons}

\begin{figure}[t]
    \centering
    \includegraphics[width=1.\linewidth]{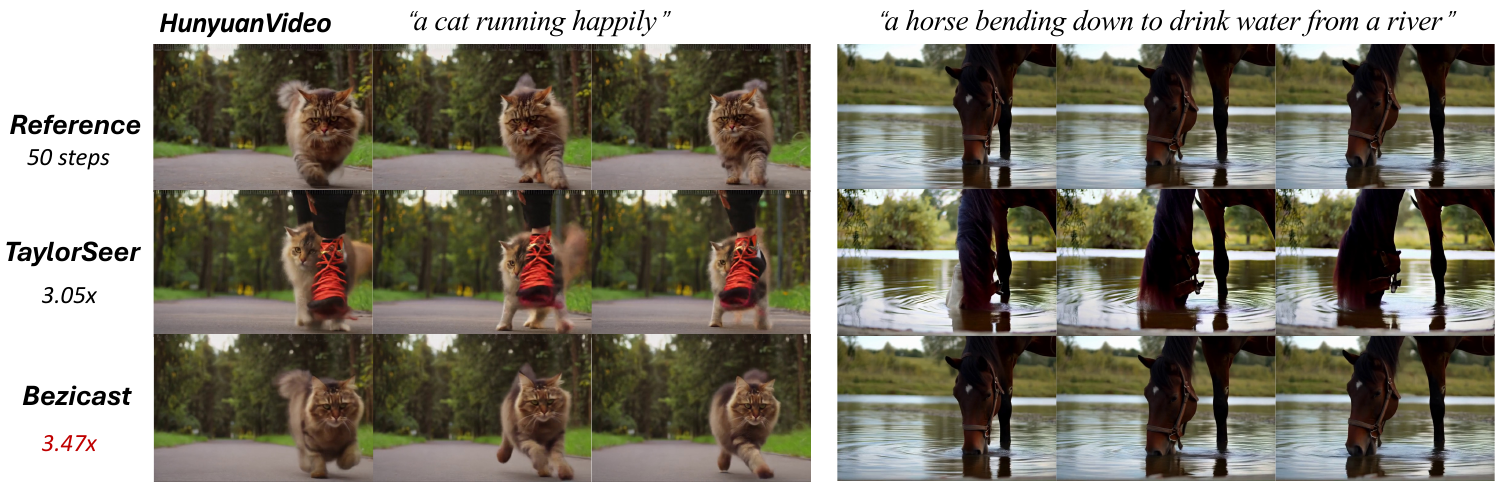}
    \vspace{-1.5em}
    \caption{
    Qualitative comparison on text-to-video generation using HunyuanVideo. BeziCast closely follow the 50-step reference, whereas TaylorSeer is slower and introduces visible artifacts.
    }
    \label{fig:hunyuan}
    \vspace{-0.6cm}
\end{figure}

\paragraph{Text-to-Image Generation.}
The left part of Table~\ref{tab:combined_t2i_t2v_comparison} reports the quantitative results on FLUX.1-dev. 
BeziCast achieves a strong trade-off between acceleration and generation quality under both caching intervals. With \(\mathcal{N}=4\), BeziCast (\(\lambda=0.5\)) uses 15 network evaluations, reaches a \(3.47\times\) speedup, and obtains the best reference-based fidelity among accelerated methods. Compared with TaylorSeer and HiCache under similar budgets, BeziCast better preserves the 50-step reference, suggesting that B\'ezier trajectory forecasting provides a more accurate approximation than Taylor-style or Hermite-based extrapolation. Under the more aggressive setting with \(\mathcal{N}=6\), BeziCast uses 12 network evaluations, achieves a \(4.79\times\) speedup, and maintains competitive quality. It outperforms direct reuse methods and forecasting baselines in reference alignment, achieving the best PSNR, SSIM, LPIPS, and ImageReward score among all \(\mathcal{N}=6\) methods. These results demonstrate the robustness of output-proximal B\'ezier forecasting for training-free text-to-image acceleration. Fig.~\ref{fig:flux_qualitative} provides a qualitative comparison. 
% BeziCast visually aligns more closely with the 50-step FLUX.1 reference, preserving text, object structure, and overall layout more faithfully than competing cache-based baselines at comparable or higher speedups.
BeziCast achieves the highest speedup while preserving prompt alignment, text legibility, and global visual fidelity comparable to the full FLUX.1 baseline.
\vspace{-0.3cm}

\paragraph{Text-to-Video Generation.}

The right part of Table~\ref{tab:combined_t2i_t2v_comparison} reports the quantitative results on HunyuanVideo using VBench.
BeziCast achieves strong text-to-video acceleration, reaching up to a \(4.11\times\) wall-clock speedup while preserving high sample fidelity. 
With \(\mathcal{N}=5\), BeziCast (\(n=4,\lambda=0.5\)) achieves a \(3.60\times\) speedup and the best overall fidelity among accelerated methods. 
Compared with TaylorSeer and HiCache at similar speedups, BeziCast better matches the full-sampling reference, suggesting that B\'ezier-based forecasting more accurately approximates latent dynamics. 
Even under the more aggressive \(\mathcal{N}=6\) setting, BeziCast maintains strong quality with a \(4.11\times\) speedup, yielding a better efficiency--fidelity trade-off than competing baselines. 
Fig.~\ref{fig:hunyuan} further shows that BeziCast more closely follows the 50-step reference and introduces fewer visible artifacts than TaylorSeer. 
%We also observe SOTA performance on class-conditional image generation with DiT-XL/2, with detailed quantitative and qualitative results provided in the appendix.

\subsection{Additional BeziCast Results and Ablations.}
\label{sec:additional_ablations}

\begin{figure}[t]
    \centering
    \includegraphics[width=0.9\linewidth]{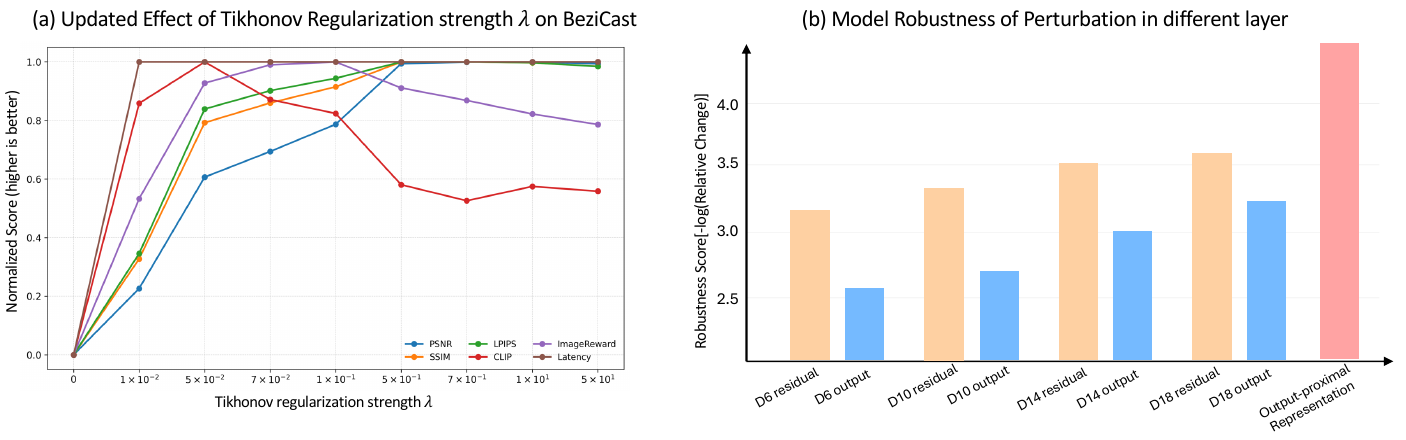}
    %\vspace{-1.5em}
    \caption{
    Ablation on Tikhonov regularization and different forecasting targets. (a)
    Effect of the Tikhonov coefficient \(\lambda\). (b) Robustness score under
    equal relative-norm perturbations injected into different FLUX
    representations; higher values indicate smaller final-output changes.
    }
    \label{fig:ablation}
    \vspace{-0.6cm}
\end{figure}

\paragraph{Effect of the ridge coefficient.}
We further study the effect of the Tikhonov coefficient \(\lambda\). As shown in
Table~\ref{tab:lambda_ablation}, small values of \(\lambda\) behave similarly to
the pseudoinverse estimator and may not sufficiently suppress extrapolation
artifacts. Increasing \(\lambda\) reduces control-point magnitude and
cache-step error, but overly large \(\lambda\) can over-shrink the trajectory,
leading to over-smoothing or loss of fine details. An intermediate
regularization strength achieves the best fidelity--stability trade-off as seen in Fig.~\ref{fig:ablation}(a).
\vspace{-0.3cm}

\begin{table}[t]
\centering
\fontsize{6.6pt}{7.5pt}\selectfont
\setlength{\tabcolsep}{1.5pt}
\renewcommand{\arraystretch}{1.0}
\newcommand{\tabstrut}{\rule[-0.75ex]{0pt}{2.65ex}}

\begin{minipage}[t]{0.485\columnwidth}
\vspace{0pt}
\centering
\caption{Effect of regularization strength \(\lambda\).}
\label{tab:lambda_ablation}
\begin{tabular*}{\linewidth}{@{\extracolsep{\fill}}c|cccccc@{}}
\toprule
\tabstrut \(\lambda\)
& PSNR \(\uparrow\)
& SSIM \(\uparrow\)
& LPIPS \(\downarrow\)
& CLIP \(\uparrow\)
& IR \(\uparrow\)
& Lat. \(\downarrow\) \\
\midrule
\tabstrut 0
& 15.425 & 0.393 & 0.523 & 26.590 & 0.758 & 15.843 \\
\tabstrut \smash{\(1{\times}10^{-2}\)}
& 17.148 & 0.524 & 0.368 & 27.331 & 0.920 & 15.208 \\
\tabstrut \smash{\(5{\times}10^{-2}\)}
& 20.043 & 0.710 & 0.147 & \textbf{27.453} & 1.040 & 15.208 \\
\tabstrut \smash{\(1{\times}10^{-1}\)}
& 21.416 & 0.759 & 0.100 & 27.301 & \textbf{1.062} & 15.208 \\
\tabstrut \smash{\(5{\times}10^{-1}\)}
& 22.993 & \textbf{0.793} & \textbf{0.075} & 27.091 & 1.035 & 15.208 \\
\tabstrut \smash{\(1{\times}10^{1}\)}
& \textbf{23.037} & 0.792 & 0.076 & 27.086 & 1.008 & 15.208 \\
\bottomrule
\end{tabular*}
\end{minipage}
\hfill
\begin{minipage}[t]{0.485\columnwidth}
\vspace{0pt}
\centering
\caption{Effect of the control point number \(n\).}
\label{tab:control_points_ablation}
\begin{tabular*}{\linewidth}{@{\extracolsep{\fill}}cc|cccccc@{}}
\toprule
\tabstrut \(W\) & \(n\)
& PSNR \(\uparrow\)
& SSIM \(\uparrow\)
& LPIPS \(\downarrow\)
& CLIP \(\uparrow\)
& IR \(\uparrow\)
& Lat. \(\downarrow\) \\
\midrule
\tabstrut 6 & 4 & 20.993 & 0.750 & 0.115 & 27.192 & \textbf{1.028} & 10.109 \\
\tabstrut 6 & 5 & 21.340 & \textbf{0.753} & \textbf{0.113} & \textbf{27.237} & 1.018 & \textbf{10.101} \\
\tabstrut 6 & 6 & \textbf{21.401} & 0.752 & \textbf{0.113} & 27.201 & 1.015 & \textbf{10.101} \\
\tabstrut 8 & 4 & 23.287 & \textbf{0.798} & \textbf{0.074} & 27.099 & 1.038 & \textbf{11.530} \\
\tabstrut 8 & 5 & 23.384 & \textbf{0.798} & 0.075 & 27.125 & 1.027 & 11.587 \\
\tabstrut 8 & 6 & \textbf{23.415} & 0.796 & 0.076 & \textbf{27.127} & \textbf{1.040} & 11.595 \\
\bottomrule
\end{tabular*}
\end{minipage}
\vspace{-0.6cm}
\end{table}

\paragraph{Perturbation Injection Diagnostic.}
We further apply the perturbation-injection diagnostic from Appendix~\ref{app:output_proximal_analysis} to measure \textit{\textbf{how errors at different forecasting targets affect the final denoising output}}. Following the local error-propagation analysis in Appendix~\ref{app:output_proximal_analysis}, intermediate-target errors must traverse the remaining downstream blocks, while output-proximal errors only pass through the final output map.
As shown in Fig.~\ref{fig:ablation}(b), the final output-proximal representation achieves the highest robustness score among all tested targets, including double-stream block residuals and outputs at layers 6, 10, 14, and 18. Intermediate residuals are generally more robust than block outputs, and later layers are more robust than earlier ones, but all remain less robust than the output-proximal representation. This supports our target choice: BeziCast does not require the final hidden state to be easier to predict, but benefits from its shorter error-propagation path to the denoising output.
\vspace{-0.4cm}

\paragraph{Equivalent-weight convex-hull guardrail}
\label{sec:ablation_guardrail}
\begin{wraptable}{r}{0.47\linewidth}
\vspace{-1.0em}
\centering
\scriptsize
\setlength{\tabcolsep}{1.6pt}
\renewcommand{\arraystretch}{0.82}
\caption{Effect of equivalent-weight convex-hull guardrails in the Flux.1.}
\label{tab:guardrail_ablation}
\vspace{-0.5em}
\resizebox{\linewidth}{!}{
\begin{tabular}{lcccccc}
\toprule
Variant 
& PSNR $\uparrow$ 
& SSIM $\uparrow$ 
& LPIPS $\downarrow$ 
& CLIP $\uparrow$ 
& Rwd. $\uparrow$ 
& Lat. $\downarrow$ \\
\midrule
None 
& 20.993 & 0.750 & 0.115 & 27.192 & 1.028 & \textbf{10.109} \\
PCA hull 
& 20.930 & 0.750 & 0.115 & 27.214 & 1.031 & 10.150 \\
Simplex 
& \textbf{21.229} 
& \textbf{0.753} 
& \textbf{0.112} 
& \textbf{27.229} 
& \textbf{1.032} 
& 10.137 \\
\bottomrule
\end{tabular}}
\vspace{-0.5cm}
\end{wraptable}
BeziCast uses equivalent history weights to diagnose extrapolation risk and constructs a conservative estimate by projecting high-risk weights onto the probability simplex. We compare three variants under the \(W=6\) FLUX schedule: no guardrail, a PCA convex-hull guardrail, and the simplex guardrail used in our main method. The PCA variant projects cached output-proximal representations into a low-dimensional PCA subspace, applies a convex-hull correction, and maps the estimate back to the original representation space. While lightweight, its behavior depends on the retained PCA subspace.
Table~\ref{tab:guardrail_ablation} shows that the simplex guardrail provides the most stable reference-fidelity improvement with only minor latency overhead. In contrast, the PCA convex-hull guardrail slightly improves some semantic metrics but yields less consistent reference-fidelity gains. These results support using the simplex guardrail by default: rather than reinterpreting \(u>1\) B\'ezier extrapolation as standard convex-hull interpolation, it pulls high-risk forecasts toward a conservative estimate when the weights contain large negative entries or deviate from the affine constraint.
\vspace{-0.4cm}

\paragraph{Number of control representations}
\label{sec:ablation_control_points}

% The number of control representations \(n\) sets the polynomial degree and
% capacity of the Bernstein/B\'ezier trajectory. In contrast, the warm-up length
% \(W\) controls how many initial solver steps are evaluated by the full denoiser
% before cache-step forecasting begins. We ablate \(n\in\{4,5,6\}\) under two
% warm-up settings, \(W=6\) and \(W=8\), while keeping the Tikhonov coefficient
% fixed at \(\lambda=5\times10^{-1}\). This experiment tests whether increasing
% trajectory capacity improves forecasting once the cache schedule is fixed.
% Table~\ref{tab:control_points_ablation} shows that increasing \(n\) gives only a
% small PSNR gain. Under \(W=6\), PSNR improves from 20.993 at \(n=4\) to 21.401
% at \(n=6\), while SSIM, LPIPS, CLIP, and ImageReward change only slightly. Under
% \(W=8\), PSNR similarly increases from 23.287 to 23.415, but SSIM and LPIPS do
% not improve. Larger \(n\) also adds a small latency cost in the \(W=8\) setting.
% We therefore use \(n=4\) as the default setting: it gives a stable cubic
% trajectory with competitive fidelity and the lowest or near-lowest latency,
% while higher-order trajectories provide only limited additional capacity in
% these schedules.

%压缩后
The number of control representations \(n\) determines the degree and capacity of the B\'ezier trajectory, while the warm-up length \(W\) controls how many initial solver steps are fully evaluated before forecasting begins. We ablate \(n\in\{4,5,6\}\) under \(W=6\) and \(W=8\), with the Tikhonov coefficient fixed at \(\lambda=5\times10^{-1}\).
As shown in Table~\ref{tab:control_points_ablation}, increasing \(n\) yields only marginal gains in reference fidelity, while the other quality metrics change little and latency may slightly increase. We therefore use \(n=4\) as the default setting, as it provides a stable cubic trajectory with competitive quality and the lowest or near-lowest latency.
\vspace{-0.4cm}

% \subsection{Discussion and Limitations}
% \label{sec:exp_discussion}

% The current results support the main design motivations of BeziCast. On FLUX, replacing TaylorSeer finite-difference forecasting with Tikhonov-stabilized Bernstein/B\'ezier forecasting improves PSNR, SSIM, and LPIPS under matched local cache schedules and usually reduces latency. On HunyuanVideo, the completed first5 and first6 settings improve both reference-fidelity metrics and VBench quality, suggesting that output-proximal B\'ezier forecasting can transfer to video generation. On DiT-XL/2, BeziCast improves FID and sFID over the local TaylorSeer baseline, although the current guardrail implementations introduce substantial overhead and require optimization.

% The main limitation of the current experimental section is protocol completeness. Several BeziCast rows have pending FLOPs or latency values, and some appended rows use local hardware/protocol settings rather than the exact HiCache reporting protocol. We therefore use the appended rows primarily to show matched comparisons against our TaylorSeer baselines and reserve fully controlled cross-paper claims until latency and FLOPs logs are finalized.

\section{Conclusion}
\vspace{-0.25cm}

% We presented BeziCast, a training-free framework for accelerating diffusion sampling by forecasting output-proximal denoising representations with low-order B\'ezier trajectories. By combining Bernstein fitting with a lightweight guardrail, BeziCast preserves the pretrained denoiser and sampler while improving the speed--quality trade-off across image and video generation tasks. Experiments on FLUX.1, HunyuanVideo, and class-conditional diffusion show that BeziCast achieves substantial acceleration with better fidelity than cache-based baselines, highlighting trajectory forecasting as a practical direction for diffusion acceleration.

We presented BeziCast, a training-free diffusion acceleration framework that forecasts output-proximal denoising representations with low-order B\'ezier trajectories. With Bernstein fitting and a guardrail, BeziCast preserves the pretrained denoiser and sampler while improving the speed--quality trade-off on image and video generation. Experiments on FLUX.1, HunyuanVideo demonstrate substantial acceleration and stronger fidelity than cache-based baselines.

%% file: Contents/appendix.tex
\appendix
\section{Technical appendices and supplementary material}
\label{app:appendix}

\subsection{Why Use a Tikhonov-Stabilized Bernstein Basis?}
\label{app:bernstein_stability}

\subsubsection{Design Motivation of Tikhonov-Stabilized Bernstein Basis}
\label{app:bernstein_design_motivation}

The temporal basis determines how a cache-then-forecast method turns a short
history of computed denoising representations into predictions at skipped
timesteps~\cite{liu2025reusing,fenghicache,han2026spectrum}. Taylor-style
predictors estimate local finite differences and extrapolate them in a
Newton-form power basis~\cite{liu2025reusing,fenghicache}.
This design is effective for short horizons, but it couples two sources of
sensitivity: finite differences first form alternating high-order derivative
estimates from nearby cached features, and extrapolation then multiplies these
estimates by coefficients that grow with the forecast offset and expansion
order. Proposition A.1 in Appendix~\ref{app:bernstein_stability} makes this
mechanism explicit: for an
order-\(p\) backward finite-difference predictor at future offset \(q\), bounded
cache perturbations are amplified by a factor upper bounded by
\(\sum_{j=0}^{p}|\binom{q}{j}|2^j\) in
Eq.~\eqref{eq:app_taylor_perturb_bound}, whose leading term grows as a
degree-\(p\) polynomial in \(q\).

BeziCast replaces direct reuse and high-order finite-difference
extrapolation with a regularized local trajectory fit. It estimates a low-order
Bernstein/B\'ezier curve from recent output-proximal representations and
evaluates the fitted curve at cache-step coordinates. A B\'ezier curve is
defined by a small set of control points and Bernstein basis
functions~\cite{farouki1987numerical}; once the control points are fixed, the
curve can be evaluated analytically at any normalized time
coordinate~\cite{patrikalakis2009shape}. Thus, the number of control points
sets trajectory capacity, while the cache schedule only determines how often
the trajectory is queried.

Cache forecasting differs from standard B\'ezier interpolation because future
full-step representations are unavailable during online inference. At a cache
step, the method observes only previously computed full steps, so the query is
often a short-range extrapolation beyond the most recent full step, i.e.,
\(u>1\). We therefore do not use the in-interval convex-hull property of
B\'ezier curves as a direct correctness guarantee. Instead, we use the
Bernstein/B\'ezier form as a smooth low-order trajectory family and explicitly
control extrapolation through Tikhonov regularization, query-range clipping,
and the equivalent-weight guardrail introduced in
Section~\ref{sec:convex_hull_guardrail}.

The Bernstein basis provides the geometric structure behind this design. On
the standard interval \(u\in[0,1]\), Bernstein basis functions are nonnegative
and form a partition of unity, so the corresponding B\'ezier curve lies in the
convex hull of its control points~\cite{patrikalakis2009shape}. Farouki and
Goodman further show that the Bernstein basis has optimal stability among
nonnegative polynomial bases on a fixed interval~\cite{farouki1996optimal}.
Although cache prediction may query \(u>1\), these properties remain useful as
a diagnostic lens: after Tikhonov fitting, each prediction can be written as an
equivalent linear combination of historical representations. If the equivalent
weights become heavily signed or deviate from an affine combination, the
forecast is relying on aggressive extrapolation rather than a conservative
combination of observed features. This observation motivates both the
regularized fitting below and the coefficient-space guardrail that follows~\cite{li2025beyond,li2024freestyleret}.

\subsubsection{External Propositions and Proofs}

This section formalizes the temporal-basis motivation in Section~\ref{sec:bezier_design} and Appendix~\ref{app:bernstein_design_motivation}. We first derive a perturbation bound for Taylor-style finite-difference extrapolation, and then show that the ridge-stabilized Bernstein/B\'ezier predictor admits a regularization-controlled stability bound.

\paragraph{Perturbation amplification in Taylor-style forecasting.}
Consider a scalar feature trajectory observed at equally spaced historical full steps,
\begin{equation}
    z_0,z_{-1},\ldots,z_{-p}.
\end{equation}
The vector-valued case follows by applying the same argument coordinate-wise. For a future offset \(q>0\), write the order-\(p\) backward finite-difference predictor as
\begin{equation}
    \hat z_{\mathrm{Tay}}(q)
    =
    \sum_{j=0}^{p}
    \binom{q}{j}\nabla^j z_0,
    \label{eq:app_taylor_newton}
\end{equation}
where \(\nabla^j\) denotes the \(j\)-th backward difference.

\noindent\textbf{Proposition A.1 (Taylor perturbation).}
If each cached value is perturbed by \(\eta_{-r}\) with
\(|\eta_{-r}|\le\varepsilon\), then the induced perturbation of
Eq.~\eqref{eq:app_taylor_newton} satisfies
\begin{equation}
    |\delta \hat z_{\mathrm{Tay}}(q)|
    \le
    \varepsilon
    \sum_{j=0}^{p}
    \left|\binom{q}{j}\right|2^j.
    \label{eq:app_taylor_perturb_bound}
\end{equation}
For fixed order \(p\), the leading factor in this upper bound grows as a
degree-\(p\) polynomial in the forecast offset \(q\).

\noindent\emph{Proof.}
The \(j\)-th backward difference can be expanded as
\begin{equation}
    \nabla^j z_0
    =
    \sum_{r=0}^{j}
    (-1)^r\binom{j}{r}z_{-r},
\end{equation}
so Eq.~\eqref{eq:app_taylor_newton} can be written as a linear extrapolation over
historical values,
\begin{equation}
    \hat z_{\mathrm{Tay}}(q)
    =
    \sum_{r=0}^{p}
    w_r^{\mathrm{Tay}}(q)z_{-r}.
\end{equation}
By linearity, the perturbation is
\(\delta\hat z_{\mathrm{Tay}}(q)=\sum_{r=0}^{p}w_r^{\mathrm{Tay}}(q)\eta_{-r}\).
Therefore,
\begin{equation}
    |\delta \hat z_{\mathrm{Tay}}(q)|
    \le
    \Lambda_p^{\mathrm{Tay}}(q)\varepsilon,
    \qquad
    \Lambda_p^{\mathrm{Tay}}(q)
    =
    \sum_{r=0}^{p}|w_r^{\mathrm{Tay}}(q)|.
\end{equation}
Using the triangle inequality before collecting the coefficients gives
\begin{equation}
    \Lambda_p^{\mathrm{Tay}}(q)
    \le
    \sum_{j=0}^{p}
    \left|\binom{q}{j}\right|2^j.
\end{equation}
The asymptotic statement follows because
\(\binom{q}{p}\) is a degree-\(p\) polynomial in \(q\).

This proposition identifies a perturbation mechanism rather than a universal
failure mode. Taylor-style predictors can work well for short extrapolation
offsets or low expansion orders, but the finite-difference form allows small
cache errors to be magnified when the cache interval is aggressive.

\paragraph{Stabilized Bernstein/B\'ezier forecasting.}
BeziCast instead estimates a low-order Bernstein/B\'ezier trajectory
from cached output-proximal representations. Let
\begin{equation}
    Z=
    \begin{bmatrix}
    z_{i_1}^{\top}\\
    \vdots\\
    z_{i_K}^{\top}
    \end{bmatrix}
    \in\mathbb{R}^{K\times D}
\end{equation}
stack the \(K\) cached representations. Let \(n\) be the number of B\'ezier
control representations, so the polynomial degree is \(n-1\). For normalized
coordinate \(u\), define the Bernstein basis vector \(b(u)\in\mathbb{R}^n\) by
\begin{equation}
    b_{j,n-1}(u)
    =
    \binom{n-1}{j}
    u^j(1-u)^{n-1-j},
    \qquad
    j=0,\ldots,n-1,
\end{equation}
\begin{equation}
    b(u)
    =
    [b_{0,n-1}(u),\ldots,b_{n-1,n-1}(u)]^{\top}
    \in\mathbb{R}^{n}.
\end{equation}
Let the Bernstein design matrix \(A\in\mathbb{R}^{K\times n}\) be
\begin{equation}
    A_{m,j}=b_{j,n-1}(u_m)
\end{equation}
for historical coordinates \(u_1,\ldots,u_K\). The control matrix
\(P\in\mathbb{R}^{n\times D}\) is estimated by Tikhonov-regularized least
squares,
\begin{equation}
    P_\lambda^\star
    =
    \arg\min_P
    \|AP-Z\|_F^2+
    \lambda\|\Gamma P\|_F^2,
    \label{eq:app_tikhonov_objective}
\end{equation}
where \(\lambda>0\) controls the regularization strength and \(\Gamma\) is the
Tikhonov matrix. If \(A^\top A+\lambda\Gamma^\top\Gamma\) is invertible, the
closed-form solution of Eq.~\eqref{eq:app_tikhonov_objective} is
\begin{equation}
    P_\lambda^\star
    =
    \left(A^\top A+\lambda\Gamma^\top\Gamma\right)^{-1}
    A^\top Z.
    \label{eq:app_tikhonov_solution_general}
\end{equation}
In the main implementation, we use \(\Gamma=I\), yielding the ridge-stabilized
solution
\begin{equation}
    P_\lambda^\star
    =
    (A^\top A+\lambda I)^{-1}A^\top Z.
    \label{eq:app_ridge_solution}
\end{equation}
Since \(A^\top A+\lambda I\) is positive definite for \(\lambda>0\), this
estimator remains well posed even when \(A\) is rank-deficient or
ill-conditioned. The prediction at query coordinate \(u\) is
\begin{equation}
    \hat z_{\mathrm{Bez}}(u)
    =
    b(u)^\top P_\lambda^\star
    =
    b(u)^\top(A^\top A+\lambda I)^{-1}A^\top Z.
    \label{eq:app_bezier_prediction}
\end{equation}

\noindent\textbf{Proposition A.2 (ridge Bernstein stability).}
For the ridge-stabilized predictor in Eq.~\eqref{eq:app_bezier_prediction},
any cache perturbation \(\delta Z\) induces
\begin{equation}
    \|\delta \hat z_{\mathrm{Bez}}(u)\|_F
    \le
    \frac{\|b(u)\|_2}{2\sqrt{\lambda}}
    \|\delta Z\|_F.
    \label{eq:app_bezier_perturb_general}
\end{equation}
Moreover, for the clipped extrapolation range \(u\in[0,u_{\max}]\), define
\begin{equation}
    C_b(u_{\max},n)
    =
    \sup_{u\in[0,u_{\max}]}
    \|b(u)\|_2.
\end{equation}
Then
\begin{equation}
    \|\delta \hat z_{\mathrm{Bez}}(u)\|_F
    \le
    \frac{C_b(u_{\max},n)}{2\sqrt{\lambda}}
    \|\delta Z\|_F,
    \qquad
    u\in[0,u_{\max}].
    \label{eq:app_bezier_perturb_clipped}
\end{equation}

\noindent\emph{Proof.}
Because Eq.~\eqref{eq:app_bezier_prediction} is linear in \(Z\),
\begin{equation}
    \delta \hat z_{\mathrm{Bez}}(u)
    =
    b(u)^\top(A^\top A+\lambda I)^{-1}A^\top\delta Z.
\end{equation}
Hence
\begin{equation}
    \|\delta \hat z_{\mathrm{Bez}}(u)\|_F
    \le
    \|b(u)\|_2
    \left\|
    (A^\top A+\lambda I)^{-1}A^\top
    \right\|_2
    \|\delta Z\|_F.
\end{equation}
Let \(A=U\Sigma V^\top\) be the singular value decomposition of \(A\). Then
\begin{equation}
    \left\|
    (A^\top A+\lambda I)^{-1}A^\top
    \right\|_2
    =
    \max_i
    \frac{\sigma_i}{\sigma_i^2+\lambda}
    \le
    \frac{1}{2\sqrt{\lambda}},
    \qquad \lambda>0.
\end{equation}
This proves Eq.~\eqref{eq:app_bezier_perturb_general}; taking the supremum of
\(\|b(u)\|_2\) over \(u\in[0,u_{\max}]\) gives
Eq.~\eqref{eq:app_bezier_perturb_clipped}.

The bound gives the stability rationale for the Tikhonov-stabilized
Bernstein/B\'ezier predictor. Its perturbation sensitivity is controlled by the
regularization coefficient and the bounded query range, rather than by
finite-difference order or a polynomial factor in the forecast offset. This
does not imply that Bernstein/B\'ezier forecasting is always more accurate than
Taylor forecasting; it states why the stabilized predictor is appropriate under
aggressive cache intervals.

\paragraph{Why not ordinary least squares?}
The preceding bound depends on the Tikhonov term. Without this term, the
Bernstein control points are estimated by ordinary least squares (OLS),
\begin{equation}
    P_{\mathrm{OLS}}^\star=A^+Z,
    \label{eq:app_ols_solution}
\end{equation}
where \(A^+\) is the Moore--Penrose pseudoinverse. When \(A\) has full column
rank, Eq.~\eqref{eq:app_ols_solution} reduces to
\begin{equation}
    P_{\mathrm{OLS}}^\star=(A^\top A)^{-1}A^\top Z.
\end{equation}
The induced prediction at query coordinate \(u\) is
\begin{equation}
    \hat z_{\mathrm{OLS}}(u)
    =
    b(u)^\top A^+Z.
\end{equation}

\noindent\textbf{Proposition A.3 (OLS versus ridge sensitivity).}
Fix the Bernstein design matrix \(A\) and query coordinate \(u\). Let
\(A=U\Sigma V^\top\) have rank \(r\) and nonzero singular values
\(\sigma_1\ge\cdots\ge\sigma_r>0\). For any cache perturbation \(\delta Z\),
the OLS predictor satisfies
\begin{equation}
    \|\delta \hat z_{\mathrm{OLS}}(u)\|_F
    \le
    \frac{\|b(u)\|_2}{\sigma_r}
    \|\delta Z\|_F,
    \label{eq:app_ols_perturb_bound}
\end{equation}
whereas the ridge-stabilized predictor with \(\lambda>0\) satisfies
\begin{equation}
    \|\delta \hat z_{\lambda}(u)\|_F
    \le
    \frac{\|b(u)\|_2}{2\sqrt{\lambda}}
    \|\delta Z\|_F.
    \label{eq:app_ridge_perturb_bound}
\end{equation}

\noindent\emph{Proof.}
For OLS, linearity gives
\begin{equation}
    \delta \hat z_{\mathrm{OLS}}(u)
    =
    b(u)^\top A^+\delta Z.
\end{equation}
Using submultiplicativity of the Frobenius norm,
\begin{equation}
    \|\delta \hat z_{\mathrm{OLS}}(u)\|_F
    \le
    \|b(u)\|_2\|A^+\|_2\|\delta Z\|_F.
\end{equation}
Since \(\|A^+\|_2=1/\sigma_r\), Eq.~\eqref{eq:app_ols_perturb_bound} follows.
For the ridge estimator,
\begin{equation}
    P_{\lambda}^\star
    =
    (A^\top A+\lambda I)^{-1}A^\top Z,
\end{equation}
and the same argument yields
\begin{equation}
    \|\delta \hat z_{\lambda}(u)\|_F
    \le
    \|b(u)\|_2
    \left\|(A^\top A+\lambda I)^{-1}A^\top\right\|_2
    \|\delta Z\|_F.
\end{equation}
In the singular-vector basis, the spectral filter of
\((A^\top A+\lambda I)^{-1}A^\top\) is
\begin{equation}
    g_\lambda(\sigma_i)
    =
    \frac{\sigma_i}{\sigma_i^2+\lambda}.
\end{equation}
Therefore,
\begin{equation}
    \left\|(A^\top A+\lambda I)^{-1}A^\top\right\|_2
    =
    \max_i g_\lambda(\sigma_i)
    \le
    \sup_{\sigma\ge 0}\frac{\sigma}{\sigma^2+\lambda}
    =
    \frac{1}{2\sqrt{\lambda}},
\end{equation}
which proves Eq.~\eqref{eq:app_ridge_perturb_bound}.

This proposition isolates the role of Tikhonov stabilization. The OLS
sensitivity depends on \(1/\sigma_r\), which can be large when the Bernstein
design matrix is ill-conditioned or nearly rank-deficient. By contrast, the
ridge predictor replaces this data-dependent factor with the explicit
regularization-controlled factor \(1/(2\sqrt{\lambda})\). In each nonzero
singular direction, OLS applies the filter \(1/\sigma_i\), while ridge applies
\(\sigma_i/(\sigma_i^2+\lambda)\). Their ratio is
\begin{equation}
    \frac{\sigma_i/(\sigma_i^2+\lambda)}{1/\sigma_i}
    =
    \frac{\sigma_i^2}{\sigma_i^2+\lambda}
    \le 1,
\end{equation}
so ridge attenuates directions that are weakly observed by the cache design.
Thus, Tikhonov fitting preserves the Bernstein/B\'ezier trajectory class but
stabilizes the estimation of its control representations.

This stability is important for training-free cache-step forecasting. The
history window contains only a few full denoiser evaluations, adjacent
denoising representations are highly correlated, and cache queries are often
extrapolative (\(u>1\)). In this regime, OLS may fit small cache perturbations
or local fluctuations too closely, and these perturbations can be propagated by
extrapolative Bernstein weights. Tikhonov fitting instead trades exact
interpolation of the cached history for a smoother and better-conditioned
control-point estimate.

The same conclusion can be viewed through the standard bias--variance
decomposition. Suppose the cached representations follow a local model
\begin{equation}
    Z=AP^\star+E,
\end{equation}
where \(P^\star\) is an ideal local control-point matrix and \(E\) denotes cache
noise or approximation error. If \(A\) has full column rank, then
\begin{equation}
    P_{\mathrm{OLS}}^\star-P^\star
    =
    A^+E,
\end{equation}
whereas ridge gives
\begin{equation}
    P_{\lambda}^\star-P^\star
    =
    -\lambda(A^\top A+\lambda I)^{-1}P^\star
    +
    (A^\top A+\lambda I)^{-1}A^\top E.
\end{equation}
The first term is the shrinkage bias introduced by ridge, and the second term
is the noise-driven estimation error. The Tikhonov term damps the
small-singular-value components of the noise term at the cost of this controlled
bias. Since the goal of BeziCast is stable extrapolation at skipped
steps rather than exact reconstruction of the cached history, this trade-off is
better aligned with the inference-time forecasting objective.

\subsection{Why Forecast Output-Proximal Representations?}
\label{app:output_proximal_analysis}

This subsection explains why BeziCast forecasts an output-proximal
representation \(z_s\) rather than intermediate module outputs. We first give a
local error-propagation analysis. We then report ablations on the forecasting
target, perturbation sensitivity, and the step at which output-proximal
forecasting is activated.

\subsubsection{Theoretical Analysis}
\label{app:output_proximal_theory}

The key distinction is the error-injection point. An intermediate prediction
error must pass through the remaining denoiser blocks before reaching the final
noise estimate, whereas an output-proximal error only passes through the final
output map.

\paragraph{Intermediate-module error propagation.}
Consider an intermediate representation \(h_s^\ell\) inside the denoiser. Let
the downstream computation after layer \(\ell\) be
\begin{equation}
    h_s^{r+1}
    =
    \Phi_r(h_s^r,t_s,c),
    \qquad
    r=\ell,\ldots,L-1,
\end{equation}
with final denoising output
\begin{equation}
    \epsilon_s
    =
    \mathcal{O}_\theta(h_s^L,t_s,c).
\end{equation}

\noindent\textbf{Proposition A.4 (intermediate-module error propagation).}
Suppose the downstream blocks \(\Phi_\ell,\ldots,\Phi_{L-1}\) have local
Lipschitz constants \(L_\ell,\ldots,L_{L-1}\), and the output map
\(\mathcal{O}_\theta\) has local Lipschitz constant \(L_{\mathcal{O}}\). If the
intermediate representation is replaced by
\begin{equation}
    \hat h_s^\ell=h_s^\ell+e_s^\ell,
\end{equation}
then the induced denoising error satisfies
\begin{equation}
    \|\delta\epsilon_s\|
    \le
    L_{\mathcal{O}}
    \prod_{r=\ell}^{L-1}L_r
    \|e_s^\ell\|.
    \label{eq:app_local_error_propagation}
\end{equation}

\noindent\emph{Proof.}
The denoising error caused by replacing \(h_s^\ell\) with \(\hat h_s^\ell\) is
\begin{equation}
\begin{aligned}
    \delta\epsilon_s
    ={}&
    \mathcal{O}_\theta\circ\Phi_{L-1}\circ\cdots\circ\Phi_{\ell}
    (h_s^\ell+e_s^\ell) \\
    &-
    \mathcal{O}_\theta\circ\Phi_{L-1}\circ\cdots\circ\Phi_{\ell}
    (h_s^\ell).
\end{aligned}
\end{equation}
Applying the local Lipschitz bounds sequentially through the composed
downstream map gives Eq.~\eqref{eq:app_local_error_propagation}.

This proposition does not state that intermediate errors are always amplified.
It states that their effect on the final denoising output depends on the local
Jacobian path from the injection point to the output. Under the same local
Lipschitz assumptions, a corresponding bound for multiple predicted modules can
be obtained by accumulating their downstream-modulated terms.

\paragraph{Output-proximal error propagation.}
The output-proximal representation \(z_s\) in Eq.~\eqref{eq:denoiser_decomposition}
is placed immediately before the final output map:
\begin{equation}
    \epsilon_s
    =
    \mathcal{O}_\theta(z_s,t_s,c).
\end{equation}

\noindent\textbf{Proposition A.5 (output-proximal error propagation).}
Assume the output map \(\mathcal{O}_\theta\) has local Lipschitz constant
\(L_{\mathcal{O}}\) with respect to its representation input. If
\begin{equation}
    \hat z_s=z_s+e_s^z,
\end{equation}
then
\begin{equation}
    \|\mathcal{O}_\theta(\hat z_s,t_s,c)-\mathcal{O}_\theta(z_s,t_s,c)\|
    \le
    L_{\mathcal{O}}\|e_s^z\|.
    \label{eq:app_output_proximal_error}
\end{equation}

\noindent\emph{Proof.}
Equation~\eqref{eq:app_output_proximal_error} follows directly from the local
Lipschitz assumption on \(\mathcal{O}_\theta\).

Compared with Eq.~\eqref{eq:app_local_error_propagation},
Eq.~\eqref{eq:app_output_proximal_error} removes the additional product of
downstream block constants. This does not imply that \(z_s\) is always easier to
forecast than intermediate activations. Rather, it makes the relationship
between representation prediction error and denoising-output error more direct,
which is the property used by BeziCast.

\paragraph{Sampler-level consequence.}
The denoising error then enters the diffusion sampler. Let the full and
forecasted trajectories use the same solver update
\begin{equation}
    x_{s+1}=\Psi_s(x_s,\epsilon_s),
    \qquad
    \hat x_{s+1}=\Psi_s(\hat x_s,\hat\epsilon_s).
\end{equation}

\noindent\textbf{Proposition A.6 (sampler-level error propagation).}
Assume \(\Psi_s\) is locally Lipschitz in both arguments:
\begin{equation}
    \|\Psi_s(x,\epsilon)-\Psi_s(x',\epsilon')\|
    \le
    L_s^x\|x-x'\|+L_s^\epsilon\|\epsilon-\epsilon'\|.
\end{equation}
If both trajectories start from the same initial latent, then
\begin{equation}
    \|\hat x_T-x_T\|
    \le
    \sum_{s=0}^{T-1}
    \left(
    L_s^\epsilon
    \prod_{r=s+1}^{T-1}L_r^x
    \right)
    \|\hat\epsilon_s-\epsilon_s\|.
    \label{eq:app_sampler_error_unroll}
\end{equation}

\noindent\emph{Proof.}
Let \(d_s=\|\hat x_s-x_s\|\) and
\(\delta_s=\|\hat\epsilon_s-\epsilon_s\|\). The Lipschitz condition gives
\begin{equation}
    d_{s+1}\le L_s^x d_s+L_s^\epsilon\delta_s.
\end{equation}
Unrolling this recursion from \(d_0=0\) yields
Eq.~\eqref{eq:app_sampler_error_unroll}.

This sampler-level bound explains why reducing denoising-output error is a
natural local objective for cache-step forecasting. Output-proximal prediction
does not remove sampler error accumulation, but it gives a shorter and more
direct path from the forecasted representation to the denoising output that the
sampler actually consumes.

\subsubsection{Ablation: Forecasting Target}
\label{app:ablation_target}

We first ablate the representation used as the forecasting target. To isolate
the target choice, we keep the pretrained FLUX denoiser, sampler, cache
schedule, history length, Bernstein order, Tikhonov coefficient, and guardrail
configuration fixed, and vary only the predicted representation.

We compare three target types. \textbf{Block-wise forecasting} applies the same
B\'ezier forecaster to local transformer-block outputs, resembling local-module
feature caching where prediction errors are injected before subsequent blocks.
\textbf{Residual forecasting} predicts the block residual/update representation,
following the residual-style target used in prior caching methods. In our
implementation, the residual is defined as the difference between the block
output and the block input. \textbf{Output-proximal forecasting} predicts the
final hidden representation \(z_s\) immediately before the output map
\(\mathcal O_\theta\), which is the default design of BeziCast.

\begin{table}[t]
\centering
\caption{
Ablation on forecasting targets on FLUX. All variants use the same pretrained
denoiser, sampler, cache schedule, history length, Bernstein order, Tikhonov
coefficient, and guardrail configuration.
}
\label{tab:app_target_ablation}
\resizebox{\linewidth}{!}{
\begin{tabular}{lcccccc}
\toprule
Forecasting target
& PSNR \(\uparrow\)
& SSIM \(\uparrow\)
& LPIPS \(\downarrow\)
& CLIP \(\uparrow\)
& ImageReward \(\uparrow\)
& Latency \(\downarrow\) \\
\midrule
Block-wise forecasting
& 22.993 & 0.793 & 0.075 & 27.091 & 1.035 & 15.208 \\
Residual forecasting
& \textbf{23.289} & 0.798 & 0.074 & 27.097 & 1.026 & 14.016 \\
\textbf{Output-proximal forecasting} \(z_s\)
& 23.287 & \textbf{0.798} & \textbf{0.074} & \textbf{27.099} & \textbf{1.038} & \textbf{11.530} \\
\bottomrule
\end{tabular}}
\end{table}

Table~\ref{tab:app_target_ablation} shows that output-proximal forecasting gives the
best latency while matching the reference fidelity of residual forecasting. It
has nearly identical PSNR, SSIM, and LPIPS to residual forecasting, but reduces
latency from 14.016s to 11.530s and gives the best CLIP and ImageReward scores.
Compared with block-wise forecasting, it improves all listed quality metrics and
also reduces latency. These results indicate that the final hidden
representation remains predictable enough for cache-step forecasting, while
avoiding the cost of predicting multiple intermediate targets.

\paragraph{Perturbation injection diagnostic.}
We next measure how a fixed relative representation error affects the final
denoising output. For a candidate target \(q\) with representation \(h_s^q\), let
\(\mathcal G_q\) denote the downstream computation from that representation to
the denoising output. We inject a random perturbation with relative magnitude
\(\sigma\):
\[
\tilde h_s^q
=
h_s^q+\sigma\frac{\xi}{\|\xi\|_2}\|h_s^q\|_2,
\]
where \(\xi\) is random noise and \(\sigma=0.01\). We then continue the forward
pass from \(\tilde h_s^q\) and measure
\[
\Delta_\epsilon(q)
=
\frac{
\|\mathcal G_q(\tilde h_s^q)-\mathcal G_q(h_s^q)\|_2
}{
\|\mathcal G_q(h_s^q)\|_2
}.
\]
A smaller \(\Delta_\epsilon(q)\) means that the same relative target error
induces a smaller final-output deviation.

We compare the final hidden state with double-stream block residuals and
double-stream block outputs at layers 6, 10, 14, and 18. We focus on
double-stream blocks only, to avoid mixing structurally different block types.
Figure~\ref{fig:ablation}(b) shows that the final hidden state has the
smallest output deviation under the same relative perturbation. Block residuals
are less sensitive than block outputs at the same layer, and later blocks are
generally less sensitive than earlier ones, but all probed block-level targets
remain more sensitive than the final hidden representation. This matches the
analysis above: intermediate errors pass through additional downstream
computation, while the output-proximal target has the shortest path to the
denoising output.

% \begin{figure}[t]
% \centering
% \includegraphics[width=\linewidth]{imgs/target_perturbation_diagnostic.png}
% \caption{
% Perturbation sensitivity on FLUX. Equal relative-norm perturbations
% \((\sigma=0.01)\) induce the smallest final-output change when injected into the
% final hidden state.
% }
% \label{fig:target_perturbation}
% \end{figure}

\paragraph{When to switch to output-proximal forecasting.}
Finally, we test whether output-proximal forecasting should start immediately
after warm-up or only in later denoising steps. The first 8 solver steps are
full denoiser evaluations in all variants. Before the switch step, cache steps
use the block-wise target; from the switch step onward, cache steps use the
output-proximal target \(z_s\). All variants use the same cache schedule,
\(\lambda=5\times10^{-1}\), and \(n=6\).

\begin{table}[t]
\centering
\caption{
Ablation on when to switch to output-proximal forecasting. Steps are
0-indexed; the first 8 steps are always full denoiser evaluations. Switching at
the first cache step gives the best latency and overall fidelity in this
setting.
}
\label{tab:target_switch_ablation}
\resizebox{\linewidth}{!}{
\begin{tabular}{lcccccc}
\toprule
Output-proximal switch step
& PSNR \(\uparrow\)
& SSIM \(\uparrow\)
& LPIPS \(\downarrow\)
& CLIP \(\uparrow\)
& ImageReward \(\uparrow\)
& Latency \(\downarrow\) \\
\midrule
25
& 23.199 & 0.797 & 0.074 & 27.100 & 1.031 & 13.146 \\
20
& 23.229 & 0.797 & 0.074 & \textbf{27.111} & 1.030 & 12.622 \\
15
& 23.243 & 0.797 & 0.075 & 27.101 & 1.027 & 12.176 \\
10
& 23.274 & \textbf{0.798} & 0.075 & 27.081 & 1.026 & 11.749 \\
0 / before first cache step
& \textbf{23.287} & \textbf{0.798} & \textbf{0.074} & 27.099 & \textbf{1.038} & \textbf{11.530} \\
\bottomrule
\end{tabular}}
\end{table}

Table~\ref{tab:target_switch_ablation} shows that activating
output-proximal forecasting from the beginning of the cache phase gives the
lowest latency and the strongest overall fidelity among the tested switch
points. Delaying the switch does not improve PSNR, SSIM, or LPIPS, and it
increases latency because more cache steps use the more expensive block-wise
target. This result is consistent with the target ablation: after a short
full-step warm-up, the output-proximal representation is already a suitable
forecasting target. We therefore use output-proximal forecasting from the first
cache step in the main experiments.

\subsection{Warm-up and Fixed-period Full-step Schedule}
\label{app:warmup_fixed_schedule}

This subsection first defines the deterministic full-step schedule used by
BeziCast, and then reports an ablation on the warm-up length. The schedule has
two roles: the initial warm-up provides real output-proximal representations for
the first local fit, and the later fixed-period full steps refresh the history
used by cache-step forecasting.

\subsubsection{Method Details}
\label{app:warmup_schedule_method}
BeziCast uses a warm-up plus fixed-period full-step schedule. The warm-up
length \(W\) is the number of initial solver steps evaluated by the full
denoiser, and the fixed period \(R\) specifies how often a full denoiser
evaluation is inserted after warm-up. At full steps, BeziCast computes and
caches the true output-proximal representation; at cache steps, it predicts the
current representation from the most recent cached full-step representations.

For a determined or over-determined control-point fit, the warm-up length should
provide at least as many historical full-step representations as the number of
Bernstein basis functions. Let \(n\) denote the number of B\'ezier control
representations. The default well-posed setting therefore uses
\begin{equation}
    W\ge n.
    \label{eq:warmup_requirement}
\end{equation}
In the main configuration, BeziCast uses a cubic Bernstein/B\'ezier forecaster
with \(n=4\), matching the control-representation notation used in the
experimental tables:
\begin{equation}
\begin{aligned}
    b_{0,3}(u)&=(1-u)^3, &
    b_{1,3}(u)&=3u(1-u)^2,\\
    b_{2,3}(u)&=3u^2(1-u), &
    b_{3,3}(u)&=u^3.
\end{aligned}
\end{equation}
Thus the default well-posed cubic setting uses \(W\ge4\). The ridge system is
still algebraically defined when fewer than \(n\) full-step representations are
available, because the Tikhonov term makes the normal matrix invertible. We
therefore also evaluate the aggressive \(W=3\) setting in the ablation below,
where the goal is to measure the speed--fidelity trade-off near the
well-posedness boundary.

After warm-up, full denoiser evaluations follow a fixed period. We use the same
zero-indexed solver convention as the main method, \(s\in\{0,\ldots,T-1\}\).
Let \(R\) be the full-step period after warm-up, meaning that one full denoiser
evaluation is performed every \(R\) solver steps. The full-step set is
\begin{equation}
\mathcal S_{\mathrm{full}}
=
\{0,\ldots,W-1\}
\cup
\left\{
s\in\{W,\ldots,T-1\}:
(s-W+1)\equiv 0 \pmod R
\right\},
\label{eq:fixed_period_schedule}
\end{equation}
and the cache-step set is
\begin{equation}
\mathcal S_{\mathrm{cache}}
=
\{0,\ldots,T-1\}\setminus \mathcal S_{\mathrm{full}}.
\end{equation}

At a cache step \(s\), the predictor uses only the most recent full-step
representations:
\begin{equation}
\mathcal I_s
=
\operatorname{Last}_{K_{\max}}
\left(
\{i\in\mathcal S_{\mathrm{full}}: i<s\}
\right),
\label{eq:local_history_indices}
\end{equation}
where \(K_{\max}\) is the maximum history length. Unless otherwise stated, we
set \(K_{\max}=6\). Therefore, each cache prediction uses at most six real
output-proximal representations, and the cubic B\'ezier trajectory is fitted
within a short local history window.

\subsubsection{Ablation Experiments}
\label{app:ablation_schedule}
We evaluate the warm-up schedule on FLUX.1-dev using the same DrawBench
protocol and metrics as Section~\ref{sec:experiments}. All variants use the
output-proximal target with \(\lambda=0.5\), \(n=4\), \(K_{\max}=6\), and
fixed period \(R=6\). We vary only the warm-up length
\(W\in\{3,4,5,6,7,8\}\). This isolates the effect of the number of initial full
denoiser evaluations before the first cache prediction.

\begin{table}[t]
\centering
\caption{Warm-up schedule ablation on FLUX. All rows use output-proximal BeziCast with \(R=6\), \(n=4\), \(K_{\max}=6\), and \(\lambda=0.5\).}
\label{tab:schedule_ablation}
\resizebox{\linewidth}{!}{
\begin{tabular}{lcccccc}
\toprule
Warm-up \(W\) & PSNR $\uparrow$ & SSIM $\uparrow$ & LPIPS $\downarrow$ & CLIP $\uparrow$ & ImageReward $\uparrow$ & Latency / Speedup \\
\midrule
3 & 17.070 & 0.623 & 0.242 & 27.117 & 0.984 & 7.881 / $4.78\times$ \\
4 & 18.598 & 0.678 & 0.180 & 27.310 & 1.045 & 8.732 / $4.32\times$ \\
5 & 20.010 & 0.724 & 0.135 & 27.414 & 1.046 & 9.383 / $4.02\times$ \\
6 & 20.993 & 0.750 & 0.115 & 27.192 & 1.028 & 10.109 / $3.73\times$ \\
7 & 22.206 & 0.780 & 0.093 & 27.060 & 1.028 & 10.848 / $3.47\times$ \\
8 & 23.287 & 0.798 & 0.074 & 27.099 & 1.038 & 11.530 / $3.27\times$ \\
\bottomrule
\end{tabular}}
\end{table}

Table~\ref{tab:schedule_ablation} shows a clear speed--fidelity trade-off.
Increasing \(W\) gives the forecaster a longer real history before cache-step
prediction begins and improves reference fidelity: PSNR increases from 17.070
at \(W=3\) to 23.287 at \(W=8\), SSIM increases from 0.623 to 0.798, and LPIPS
decreases from 0.242 to 0.074. The cost is higher latency, which increases from
7.881s to 11.530s and reduces speedup from \(4.78\times\) to \(3.27\times\).
CLIP and ImageReward vary in a narrower range, so this ablation mainly reflects
the trade-off between reference fidelity and runtime.

We use the aggressive \(W=3\) schedule for the \(\mathcal{N}=6\) FLUX main result (in Table~\ref{tab:combined_t2i_t2v_comparison})
because it provides the largest speedup in this ablation while still improving
the matched TaylorSeer baselines.
Specifically, under the same \(\mathcal{N}=6\) setting, BeziCast with \(W=3\)
reaches 17.070 PSNR, 0.623 SSIM, and 0.242 LPIPS at 7.881s, while the best
TaylorSeer row reports 15.872 PSNR, 0.621 SSIM, and 0.281 LPIPS at 8.862s. When
higher reference fidelity is preferred over maximum speed, increasing \(W\)
offers a direct schedule-level adjustment without changing the forecaster,
output target, or pretrained denoiser.

\subsection{Online Inference Algorithm}
\label{app:online_algorithm}

Algorithm~\ref{alg:BeziCast} summarizes inference. We first construct
the warm-up plus fixed-period full-step set. Full steps run the original
denoiser and update the cache. Cache steps compute the equivalent
ridge-B\'ezier weights, optionally blend with the convex-hull fallback, and pass
the predicted output-proximal representation through the unchanged output map.
The \(K_s<n\) branch is a fallback branch; under the default requirement
\(W\ge n\), cache prediction starts after enough full-step
representations are available.

\begin{algorithm}[t]
\caption{BeziCast inference}
\label{alg:BeziCast}
\begin{algorithmic}[1]
\Require pretrained denoiser \((\mathcal{F}_\theta,\mathcal{O}_\theta)\), sampler \(\{\Psi_s\}_{s=0}^{T-1}\), time grid \(\{t_s\}_{s=0}^{T-1}\), monotone coordinates \(\{\tau_s\}_{s=0}^{T-1}\), condition \(c\), warm-up length \(W\), fixed full-step period \(R\), maximum history length \(K_{\max}\), control-point number \(n\), ridge strength \(\lambda\), extrapolation cap \(u_{\max}\), and guardrail parameters \((\beta_{\max},\rho,\gamma)\).
\State Construct \(\mathcal{S}_{\mathrm{full}}\) by Eq.~\eqref{eq:fixed_period_schedule}.
\State Initialize cache \(\mathcal{C}\leftarrow\emptyset\).
\For{\(s=0,\ldots,T-1\)}
    \If{\(s\in\mathcal{S}_{\mathrm{full}}\)}
        \State Compute \(z_s=\mathcal{F}_\theta(x_s,t_s,c)\).
        \State Compute \(\epsilon_s=\mathcal{O}_\theta(z_s,t_s,c)\).
        \State Update \(\mathcal{C}\leftarrow\mathcal{C}\cup\{(\tau_s,z_s)\}\) and keep the most recent \(K_{\max}\) full-step entries.
    \Else
        \State Set \(K_s=\min(K_{\max},|\{i\in\mathcal{S}_{\mathrm{full}}:i<s\}|)\).
        \If{\(K_s<n\)}
            \State Compute \(z_s=\mathcal{F}_\theta(x_s,t_s,c)\).
            \State Compute \(\epsilon_s=\mathcal{O}_\theta(z_s,t_s,c)\).
            \State Update \(\mathcal{S}_{\mathrm{full}}\leftarrow\mathcal{S}_{\mathrm{full}}\cup\{s\}\).
            \State Update \(\mathcal{C}\leftarrow\mathcal{C}\cup\{(\tau_s,z_s)\}\) and keep the most recent \(K_{\max}\) full-step entries.
        \Else
            \State Retrieve the most recent \(K_s\) full-step entries \(\{(\tau_{i_m},z_{i_m})\}_{m=1}^{K_s}\) from \(\mathcal{C}\).
            \State Compute \(\{u_m\}_{m=1}^{K_s}\) and \(u_s\) using Eqs.~\eqref{eq:history_normalization}--\eqref{eq:query_normalization}.
            \State Construct \(A\in\mathbb{R}^{K_s\times n}\) and \(Z\in\mathbb{R}^{K_s\times D}\).
            \State Compute \(\omega_\lambda(u_s)\) using Eq.~\eqref{eq:equivalent_weight}.
            \State Predict \(\hat z_s=\omega_\lambda(u_s)^\top Z\).
            \State Compute \(r(u_s)\) and \(\beta_s\) using Eqs.~\eqref{eq:risk_score} and \eqref{eq:beta_schedule}.
            \State Project \(w=\omega_\lambda(u_s)\) onto \(\Delta\) to obtain \(w_{\mathrm{safe}}\).
            \State Compute \(z_s^{\mathrm{safe}}=w_{\mathrm{safe}}^\top Z\).
            \State Set \(\tilde z_s=(1-\beta_s)\hat z_s+\beta_s z_s^{\mathrm{safe}}\).
            \State Compute \(\epsilon_s=\mathcal{O}_\theta(\tilde z_s,t_s,c)\).
        \EndIf
    \EndIf
    \State Update \(x_{s+1}=\Psi_s(x_s,\epsilon_s)\).
\EndFor
\State \Return \(x_T\).
\end{algorithmic}
\end{algorithm}

\subsection{Basis Dependence of Tikhonov-Regularized Forecasting}
\label{app:basis_dependence}

We clarify why two polynomial bases spanning the same finite-dimensional
function space can nevertheless induce different regularized forecasters.

Let
\[
\boldsymbol{\phi}(u)
=
[\phi_0(u),\ldots,\phi_{n-1}(u)]^\top
\quad\text{and}\quad
\boldsymbol{\psi}(u)
=
[\psi_0(u),\ldots,\psi_{n-1}(u)]^\top
\]
be two bases of the same polynomial space $\mathcal P_{n-1}$.
Because both are bases of $\mathcal P_{n-1}$, there exists an invertible
matrix $M\in\mathbb R^{n\times n}$ such that
\begin{equation}
    \boldsymbol{\psi}(u)=M^\top\boldsymbol{\phi}(u).
    \label{eq:basis_transform}
\end{equation}
For historical coordinates $\{u_m\}_{m=1}^{K}$, define the corresponding
design matrices
\begin{equation}
    A_{\phi}
    =
    \begin{bmatrix}
        \boldsymbol{\phi}(u_1)^\top\\
        \vdots\\
        \boldsymbol{\phi}(u_K)^\top
    \end{bmatrix},
    \qquad
    A_{\psi}
    =
    \begin{bmatrix}
        \boldsymbol{\psi}(u_1)^\top\\
        \vdots\\
        \boldsymbol{\psi}(u_K)^\top
    \end{bmatrix}.
\end{equation}
Equation~\eqref{eq:basis_transform} implies
\begin{equation}
    A_{\psi}=A_{\phi}M.
    \label{eq:design_transform}
\end{equation}

\begin{proposition}[Basis invariance without regularization]
\label{prop:ols_basis_invariance}
Assume that $A_{\phi}$ has full column rank. Let
$Z\in\mathbb R^{K\times D}$ be the matrix of cached representations.
For $\lambda=0$, the least-squares predictors constructed in the two
bases are identical at every query coordinate $u$:
\begin{equation}
\begin{aligned}
    &\boldsymbol{\phi}(u)^\top
    (A_{\phi}^\top A_{\phi})^{-1}
    A_{\phi}^\top Z\\
    &\qquad =
    \boldsymbol{\psi}(u)^\top
    (A_{\psi}^\top A_{\psi})^{-1}
    A_{\psi}^\top Z.
\end{aligned}
\label{eq:ols_basis_invariance}
\end{equation}
\end{proposition}

\begin{proof}
Let $G=A_{\phi}^\top A_{\phi}$. By
Eq.~\eqref{eq:basis_transform} and Eq.~\eqref{eq:design_transform},
\[
\boldsymbol{\psi}(u)^\top
=
\boldsymbol{\phi}(u)^\top M,
\qquad
A_{\psi}^\top A_{\psi}
=
M^\top G M.
\]
Since $M$ is invertible,
\[
(M^\top G M)^{-1}
=
M^{-1}G^{-1}M^{-\top}.
\]
Therefore,
\begin{align}
&\boldsymbol{\psi}(u)^\top
(A_{\psi}^\top A_{\psi})^{-1}
A_{\psi}^\top Z \nonumber\\
&=
\boldsymbol{\phi}(u)^\top M
(M^\top G M)^{-1}
M^\top A_{\phi}^\top Z \nonumber\\
&=
\boldsymbol{\phi}(u)^\top M
M^{-1}G^{-1}M^{-\top}
M^\top A_{\phi}^\top Z \nonumber\\
&=
\boldsymbol{\phi}(u)^\top
G^{-1}A_{\phi}^\top Z,
\end{align}
which is precisely the predictor constructed in the
$\boldsymbol{\phi}$ basis.
\end{proof}

Proposition~\ref{prop:ols_basis_invariance} shows that changing the
basis alone does not change the unregularized polynomial least-squares
fit, provided that both bases span the same space and the linear
systems are solved exactly. The situation changes when the same
identity Tikhonov penalty is imposed in both coefficient systems.

\begin{proposition}[Basis dependence of identity Tikhonov regularization]
\label{prop:tikhonov_basis_dependence}
Under the assumptions of
Proposition~\ref{prop:ols_basis_invariance}, let $\lambda>0$.
Define the identity-Tikhonov predictors
\begin{equation}
    \widehat z_{\phi,\lambda}(u)
    =
    \boldsymbol{\phi}(u)^\top
    (A_{\phi}^\top A_{\phi}+\lambda I)^{-1}
    A_{\phi}^\top Z
\end{equation}
and
\begin{equation}
    \widehat z_{\psi,\lambda}(u)
    =
    \boldsymbol{\psi}(u)^\top
    (A_{\psi}^\top A_{\psi}+\lambda I)^{-1}
    A_{\psi}^\top Z.
\end{equation}
Expressed in the $\boldsymbol{\phi}$ coordinates, the second predictor is
\begin{equation}
    \widehat z_{\psi,\lambda}(u)
    =
    \boldsymbol{\phi}(u)^\top
    \left(
        A_{\phi}^\top A_{\phi}
        +
        \lambda M^{-\top}M^{-1}
    \right)^{-1}
    A_{\phi}^\top Z.
    \label{eq:transformed_ridge_predictor}
\end{equation}
Consequently, the two identity-Tikhonov coefficient estimators coincide
for every $Z$ if and only if
\begin{equation}
    M^{-\top}M^{-1}=I,
\end{equation}
or equivalently, $M$ is orthogonal.
For a general non-orthogonal change of basis, the two regularized
forecasters are different despite spanning the same polynomial space.
\end{proposition}

\begin{proof}
Let $C_{\psi}\in\mathbb R^{n\times D}$ be the coefficient matrix in the
$\boldsymbol{\psi}$ basis. The identity-Tikhonov objective is
\begin{equation}
    \min_{C_{\psi}}
    \|A_{\psi}C_{\psi}-Z\|_F^2
    +
    \lambda\|C_{\psi}\|_F^2.
    \label{eq:psi_ridge}
\end{equation}
Define the equivalent coefficient matrix in the
$\boldsymbol{\phi}$ basis by
\begin{equation}
    C_{\phi}=MC_{\psi}.
\end{equation}
Using $A_{\psi}=A_{\phi}M$, we have
\[
A_{\psi}C_{\psi}=A_{\phi}C_{\phi},
\qquad
C_{\psi}=M^{-1}C_{\phi}.
\]
Thus, Eq.~\eqref{eq:psi_ridge} is equivalent to
\begin{equation}
    \min_{C_{\phi}}
    \|A_{\phi}C_{\phi}-Z\|_F^2
    +
    \lambda
    \|M^{-1}C_{\phi}\|_F^2.
\end{equation}
The regularization term can be written as
\begin{equation}
    \|M^{-1}C_{\phi}\|_F^2
    =
    \operatorname{tr}
    \left(
        C_{\phi}^\top
        M^{-\top}M^{-1}
        C_{\phi}
    \right).
\end{equation}
The corresponding normal equations are therefore
\begin{equation}
    \left(
        A_{\phi}^\top A_{\phi}
        +
        \lambda M^{-\top}M^{-1}
    \right)C_{\phi}
    =
    A_{\phi}^\top Z,
\end{equation}
which yields Eq.~\eqref{eq:transformed_ridge_predictor}.

By comparison, identity Tikhonov regularization directly in the
$\boldsymbol{\phi}$ coordinates gives
\begin{equation}
    \left(
        A_{\phi}^\top A_{\phi}
        +
        \lambda I
    \right)C_{\phi}
    =
    A_{\phi}^\top Z.
\end{equation}
Since $A_{\phi}$ has full column rank, $A_{\phi}^\top$ is surjective
onto $\mathbb R^n$. Hence the two coefficient estimators agree for every
$Z$ if and only if
\[
\left(
A_{\phi}^\top A_{\phi}
+
\lambda M^{-\top}M^{-1}
\right)^{-1}
=
\left(
A_{\phi}^\top A_{\phi}
+
\lambda I
\right)^{-1},
\]
which is equivalent to
\[
M^{-\top}M^{-1}=I.
\]
For an invertible square matrix, this condition is equivalent to
$M$ being orthogonal.
\end{proof}

\begin{corollary}[Basis-aware Tikhonov invariance]
\label{cor:basis_aware_regularizer}
Let the Tikhonov matrix in the $\boldsymbol{\phi}$ coordinates be
$\Gamma_{\phi}$. To obtain the same regularized function estimator in
the $\boldsymbol{\psi}$ coordinates, the regularizer must be transformed as
\begin{equation}
    \Gamma_{\psi}
    =
    \Gamma_{\phi}M.
    \label{eq:regularizer_transform}
\end{equation}
In particular, if $\Gamma_{\phi}=I$, then the basis-equivalent penalty
in the $\boldsymbol{\psi}$ coordinates is $\Gamma_{\psi}=M$, rather
than $I$.
\end{corollary}

\begin{proof}
Since $C_{\phi}=MC_{\psi}$,
\[
\|\Gamma_{\phi}C_{\phi}\|_F^2
=
\|\Gamma_{\phi}MC_{\psi}\|_F^2.
\]
Thus, choosing $\Gamma_{\psi}=\Gamma_{\phi}M$ makes both the data-fitting
term and the regularization term identical under the change of
coordinates.
\end{proof}

\paragraph{Application to Bernstein and shifted-Chebyshev bases.}
Let
\[
\boldsymbol b(u)
=
[B_0^3(u),B_1^3(u),B_2^3(u),B_3^3(u)]^\top
\]
be the cubic Bernstein basis on $[0,1]$, and let
\[
\boldsymbol c(u)
=
[T_0(2u-1),T_1(2u-1),T_2(2u-1),T_3(2u-1)]^\top
\]
be the shifted cubic Chebyshev basis. They are related by
\begin{equation}
    \boldsymbol c(u)=M^\top\boldsymbol b(u),
\end{equation}
where
\begin{equation}
M=
\begin{bmatrix}
1 & -1 & 1 & -1\\
1 & -\frac{1}{3} & -\frac{5}{3} & 5\\
1 & \frac{1}{3} & -\frac{5}{3} & -5\\
1 & 1 & 1 & 1
\end{bmatrix}.
\end{equation}
This transformation is not orthogonal; for example, the inner product
between its first and third columns is $-4/3$. Therefore,
\[
M^{-\top}M^{-1}\neq I.
\]
By Proposition~\ref{prop:tikhonov_basis_dependence}, applying the same
identity Tikhonov penalty and the same numerical value of $\lambda$ in
the Bernstein and shifted-Chebyshev coefficient systems generally
produces different fitted trajectories, different equivalent history
weights, and different cache-step predictions.

This distinction does not imply that one basis universally dominates
the other. It shows that, once coefficient regularization is introduced,
``the choice of basis'' and ``the choice of regularizer'' jointly define
the forecasting estimator. Two bases with the same approximation space
need not yield the same regularized forecasting operator.

\begin{proposition}[Uniform norm comparison between Bernstein and shifted-Chebyshev bases]
\label{prop:bernstein_chebyshev_basis_norm}

For an integer $d\geq 0$, define the degree-$d$ Bernstein basis vector on
the normalized coordinate $u\in\mathbb R$ as
\begin{equation}
    \mathbf b_d(u)
    =
    \left[
        B_0^d(u),\ldots,B_d^d(u)
    \right]^\top,
    \qquad
    B_j^d(u)
    =
    \binom{d}{j}u^j(1-u)^{d-j},
\end{equation}
and define the standard shifted-Chebyshev basis vector as
\begin{equation}
    \boldsymbol\phi_d(u)
    =
    \left[
        T_0(2u-1),\ldots,T_d(2u-1)
    \right]^\top,
\end{equation}
where $T_k$ denotes the Chebyshev polynomial of the first kind.

Then, for every $u\in\mathbb R$,
\begin{equation}
    \boxed{
    \|\mathbf b_d(u)\|_2
    \leq
    \|\boldsymbol\phi_d(u)\|_2
    }.
    \label{eq:uniform_basis_norm_comparison}
\end{equation}
Moreover, the inequality is strict for every $u\in\mathbb R$ whenever
$d\geq1$.
\end{proposition}

\begin{proof}
We consider separately the interpolation region $u\in[0,1]$ and the
extrapolation region $u\notin[0,1]$.

\paragraph{Case 1: \(u\in[0,1]\).}
The Bernstein basis functions are nonnegative and form a partition of
unity:
\begin{equation}
    B_j^d(u)\geq0,
    \qquad
    \sum_{j=0}^{d}B_j^d(u)=1.
\end{equation}
Therefore,
\begin{equation}
    \|\mathbf b_d(u)\|_2
    \leq
    \|\mathbf b_d(u)\|_1
    =
    1.
    \label{eq:bernstein_inside_bound}
\end{equation}
The shifted-Chebyshev vector always contains the constant component
\begin{equation}
    T_0(2u-1)=1,
\end{equation}
and hence
\begin{equation}
    \|\boldsymbol\phi_d(u)\|_2
    \geq 1.
    \label{eq:chebyshev_inside_bound}
\end{equation}
Combining Eqs.~\eqref{eq:bernstein_inside_bound} and
\eqref{eq:chebyshev_inside_bound} gives
\begin{equation}
    \|\mathbf b_d(u)\|_2
    \leq
    \|\boldsymbol\phi_d(u)\|_2.
\end{equation}

For $d\geq1$, the inequality is strict. If $u\in(0,1)$, at least two
Bernstein basis entries are positive, so
$\|\mathbf b_d(u)\|_2<\|\mathbf b_d(u)\|_1=1$. At $u\in\{0,1\}$,
$\|\mathbf b_d(u)\|_2=1$, whereas
\[
\|\boldsymbol\phi_d(u)\|_2=\sqrt{d+1}>1.
\]

\paragraph{Case 2: \(u\notin[0,1]\).}
Let
\begin{equation}
    x=|2u-1|>1.
\end{equation}
Using the binomial theorem,
\begin{align}
    \|\mathbf b_d(u)\|_2
    &\leq
    \|\mathbf b_d(u)\|_1 \nonumber\\
    &=
    \sum_{j=0}^{d}
    \binom{d}{j}
    |u|^j|1-u|^{d-j} \nonumber\\
    &=
    \bigl(|u|+|1-u|\bigr)^d \nonumber\\
    &=
    |2u-1|^d
    =
    x^d.
    \label{eq:bernstein_outside_bound}
\end{align}
The last equality follows because
$|u|+|1-u|=|2u-1|$ for $u\notin[0,1]$.

Since the shifted-Chebyshev vector contains its degree-$d$ component,
\begin{equation}
    \|\boldsymbol\phi_d(u)\|_2
    \geq
    |T_d(2u-1)|
    =
    T_d(x),
    \label{eq:chebyshev_component_bound}
\end{equation}
where we used
$|T_d(-x)|=|T_d(x)|=T_d(x)$ for $x\geq1$.

It remains to show that
\begin{equation}
    T_d(x)\geq x^d,
    \qquad x\geq1.
    \label{eq:chebyshev_power_bound}
\end{equation}
Write $x=\cosh\theta$ for some $\theta\geq0$. By the identity
$T_d(\cosh\theta)=\cosh(d\theta)$,
\begin{equation}
    T_d(x)
    =
    \cosh(d\theta)
    =
    \frac{e^{d\theta}+e^{-d\theta}}{2}.
\end{equation}
Because $y\mapsto y^d$ is convex on $\mathbb R_{+}$ for integer
$d\geq1$, Jensen's inequality gives
\begin{align}
    \frac{e^{d\theta}+e^{-d\theta}}{2}
    &\geq
    \left(
        \frac{e^\theta+e^{-\theta}}{2}
    \right)^d \nonumber\\
    &=
    \cosh^d\theta
    =
    x^d.
\end{align}
Therefore,
\begin{equation}
    \|\mathbf b_d(u)\|_2
    \leq
    x^d
    \leq
    T_d(x)
    \leq
    \|\boldsymbol\phi_d(u)\|_2.
\end{equation}

For $d\geq1$ and $u\notin[0,1]$, all $d+1$ Bernstein entries are
nonzero. Hence
$\|\mathbf b_d(u)\|_2<\|\mathbf b_d(u)\|_1$, making the final inequality
strict.

The case $d=0$ gives equality, since both basis vectors reduce to
$[1]$.
\end{proof}

\subsection{Analysis}
\label{sec:analysis}

\paragraph{Complexity.}
Let \(D\) be the flattened representation dimension, \(K_s\le K_{\max}\) the
available history length at cache step \(s\), and \(n\) the number of
B\'ezier control points. At one cache step, building the design matrix costs
\(O(K_sn)\). Solving the \(n\times n\) ridge system costs
\(O(n^3)\), and this part can be precomputed when the cache schedule and
hyperparameters are fixed. The main high-dimensional operation is the weighted
sum \(\omega_\lambda(u_s)^\top Z\), which costs \(O(K_sD)\). The guardrail adds a
simplex projection in \(K_s\) dimensions, costing \(O(K_s\log K_s)\) with a
sorting-based algorithm, and one additional weighted sum. The memory overhead
is \(O(K_{\max}D)\) for cached output-proximal representations.
Algorithm~\ref{alg:BeziCast} does not store the control
representations persistently; prediction uses equivalent historical weights. Our experiments on both FLUX and HunyuanVideo were run on NVIDIA H20 GPU.

\paragraph{Relation to existing predictors.}
Direct reuse is the zeroth-order case \(\hat z_s=z_{i_K}\). Taylor-style
forecasting estimates local finite differences and extrapolates them in a
power basis. Spectral predictors fit global polynomial bases over the denoising
trajectory. BeziCast uses a local function-approximation form: it fits
a short-history Bernstein/B\'ezier trajectory with Tikhonov regularization,
evaluates it at cache steps, and writes the result as equivalent weights over
historical output-proximal representations. The control-point number \(n\) is
chosen independently of the cache period.

\paragraph{Limitation.}
BeziCast is a training-free acceleration method and therefore preserves the pretrained denoiser and sampler, but its effectiveness depends on the temporal regularity of output-proximal denoising representations; when the denoising trajectory changes abruptly, aggressive cache intervals may still introduce prediction errors despite Tikhonov stabilization and the convex-hull guardrail. Its speed–quality trade-off is also sensitive to schedule and hyperparameters such as the warm-up length, cache period, Bézier control-point number, and regularization strength, which may require task- or model-specific tuning. In addition, our experiments mainly cover FLUX.1 for text-to-image generation and HunyuanVideo for text-to-video generation, so further validation on more architectures, samplers, resolutions, longer videos, and stronger downstream metrics is needed to fully establish generality. Finally, BeziCast assumes access to a suitable output-proximal representation and still executes the output map at cache steps, which may limit its applicability or acceleration benefit for architectures where such representations are inaccessible or the output head is non-negligible in cost.